\documentclass[twoside]{article}
\usepackage[preprint]{aistats2027}
\usepackage{amsmath,amssymb,amsfonts,amsthm,bm}
\usepackage{booktabs}
\usepackage{multirow}
\usepackage{graphicx}
\usepackage{xcolor}
\usepackage[numbers,square,sort&compress]{natbib}

\usepackage{algorithm}
\usepackage{algpseudocode}
\usepackage{xurl}
\usepackage{hyperref}
\makeatletter\let\ifpreprintversion\if@preprint\makeatother
\algrenewcommand\algorithmicindent{0.8em}
\algtext*{EndFor}
\newcommand{\R}{\mathbb{R}}
\newcommand{\E}{\mathbb{E}}
\newcommand{\Ind}{\mathbb{I}}
\newcommand{\Normal}{\mathcal{N}}
\newcommand{\Cat}{\mathrm{Categorical}}
\newcommand{\Poisson}{\mathrm{Poisson}}
\newcommand{\Uniform}{\mathrm{Uniform}}
\newcommand{\GammaDist}{\mathrm{Gamma}}
\newcommand{\InvGamma}{\mathrm{InvGamma}}

\newcommand{\diag}{\mathrm{diag}}
\newcommand{\dd}{\mathrm{d}}
\newcommand{\Tobs}{\mathcal{T}_{\mathrm{obs}}}
\newcommand{\Ttrue}{\mathcal{T}_{\mathrm{true}}}
\newcommand{\Tcand}{\mathcal{T}_{\mathrm{cand}}}
\newcommand{\Tevent}{\mathcal{T}_{\mathrm{event}}}
\newcommand{\Tall}{\mathcal{T}_{\mathrm{all}}}
\newcommand{\phat}{\hat{p}_{\phi}}

\definecolor{checkgreen}{RGB}{0,130,100}
\definecolor{crossred}{RGB}{190,50,50}
\colorlet{tablecheckcolor}{checkgreen}
\colorlet{tablecrosscolor}{crossred}
\newcommand{\tablecheck}{\textcolor{tablecheckcolor}{\ensuremath{\bm{\checkmark}}}}
\newcommand{\tablecross}{\textcolor{tablecrosscolor}{\ensuremath{\bm{\times}}}}

\newtheorem{proposition}{Proposition}

\DeclareRobustCommand{\modepropositionnumber}{\ref{prop:fnle_conditional_mode_stability}}

\newcommand{\papertitle}{Broadly Applicable Approximate MCMC for Switching Stochastic Differential Equations Using Uniformization and Time-Conditioned Factorized Neural Likelihood Estimation}

\begin{document}
\runningtitle{Broadly Applicable Approximate MCMC for SSDEs}
\runningauthor{Shion Hosoda, Michiaki Hamada}
\twocolumn[
\aistatstitle{\papertitle}
\aistatsauthor{Shion Hosoda\textsuperscript{1,*} \And Michiaki Hamada\textsuperscript{1,2,3}}
\aistatsaddress{
\textsuperscript{1}\,Waseda University\\
\textsuperscript{2}\,National Institute of Advanced Industrial Science and Technology\\
\textsuperscript{3}\,Nippon Medical School
}]
\ifnum\statePaper=1
\begingroup
\renewcommand{\thefootnote}{}
\begin{NoHyper}
\footnotetext{\textsuperscript{*}Corresponding author: \texttt{shion.hosoda@gmail.com}.}
\end{NoHyper}
\endgroup
\fi

\begin{abstract}
Switching stochastic differential equations (SSDEs) describe continuous-time dynamics whose parameters switch according to a latent regime process that follows a continuous-time Markov chain (CTMC).
By allowing dynamics to change between regimes, SSDEs represent heterogeneous system behavior and have been applied across diverse fields.
However, Bayesian inference for SSDEs remains difficult, and existing SSDE inference methods have limited applicability, with restrictions such as noise-free observations, univariate states, linear drift, or state-independent diffusion.
In this study, we propose an approximate Markov chain Monte Carlo sampler for SSDEs using uniformization and factorized neural likelihood estimation (FNLE), a simulation-based inference method.
Uniformization provides an exact representation of the CTMC but requires SDE transition densities over arbitrary time intervals.
We approximate these densities by training a time-conditioned FNLE model.
The resulting sampler is broadly applicable to SSDEs without requiring analytically tractable transition densities.
In synthetic-data experiments, our method recovered regime paths and parameters for three SSDE models for which previous methods have limited applicability.
We also applied our method to a real dataset and detected a regime transition.
\end{abstract}

\section{Introduction}

Stochastic differential equations (SDEs) provide a natural framework for modeling noisy continuous-time dynamics.
A single SDE may be insufficient in many cases because dynamics in real-world systems often change.
Switching stochastic differential equations (SSDEs) represent such dynamics by allowing the drift and diffusion functions to switch according to the regime at each time.
SSDEs have been applied in many fields, such as ecology~\citep{yan2014}, biology~\citep{hosoda2021}, and economics~\citep{li2013}.
Figure~\ref{fig:ssde_overview} illustrates an SSDE and its noisy observations.

Bayesian inference for SSDEs remains challenging.
It requires joint estimation of latent states, regimes, and model parameters from noisy observations.
Moreover, regime transitions can occur between observation times.
Due to these challenges, existing SSDE inference methods have limited applicability and are typically restricted to settings with noise-free observations, univariate states, linear drift, or state-independent diffusion~\citep{hibbah2020,kohs2021,kohs2022,stumpf-fetizon2026}.

Uniformization is a method to represent the continuous-time Markov chain (CTMC) by a discrete-time regime sequence on Poisson event times~\citep{rao2013}.
When regime-dependent emission likelihoods are available, this representation facilitates parameter inference in CTMC-based models and has been used in a range of applications~\citep{pan2016,georgoulas2018,gonccalves2022}.
For an SSDE, however, the required likelihood involves transition densities of the state process over each interval.
These densities are generally unavailable for SDEs, limiting the direct use of uniformization-based regime updates in general SSDEs.

Simulation-based inference can learn otherwise intractable densities from simulations~\citep{papamakarios2019,cranmer2020}. Neural likelihood estimation learns conditional densities from simulated parameter--data pairs, while factorized neural likelihood estimation (FNLE) exploits the Markov property to approximate trajectory likelihoods from learned local transition densities~\citep{gloeckler2024}. Standard FNLE is formulated for discrete-time Markovian sequences and is typically trained for a fixed transition interval, and it cannot be directly evaluated at the varying intervals introduced by uniformization.

In this study, we propose an approximate Markov chain Monte Carlo (MCMC) sampler for SSDEs using time-conditioned FNLE.
As a foundation, we develop an exact uniformization-based sampler for SSDEs with tractable transition densities and conditional intermediate-state sampling.
We then condition FNLE on elapsed time to approximate regime-specific transition densities over varying intervals and use sampling-importance-resampling (SIR) to approximate conditional intermediate-state sampling.
This yields a broadly applicable sampler for SSDEs without requiring analytically tractable transition densities.

Our contributions are summarized as follows:
\begin{itemize}
\item We combine time-conditioned FNLE with uniformization to construct an approximate MCMC sampler applicable to a broad range of SSDEs.
\item We show that the sampler successfully recovers regime paths and model parameters in synthetic experiments on three SSDE models spanning nonlinear drift, state-dependent diffusion, and multivariate latent states.
\item We apply the sampler to a real dataset and infer a persistent regime change.
\end{itemize}

\begin{figure}[!t]
    \centering
    \includegraphics[width=\linewidth]{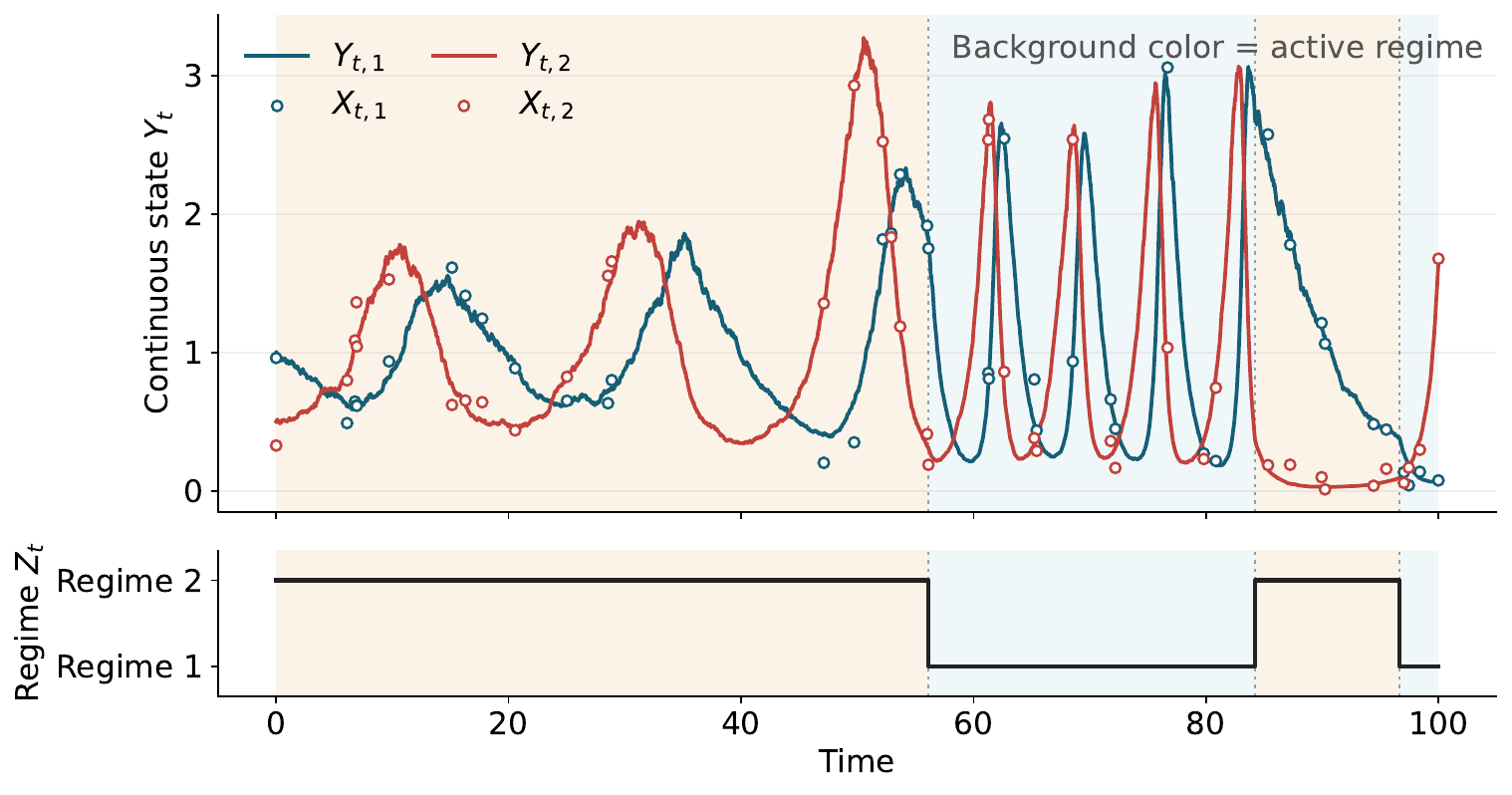}
    \caption{Illustration of a simulated Lotka--Volterra SSDE.
    The upper panel shows the two components of the continuous state and their noisy observations.
    Background shading indicates the active regime, and the lower panel shows the regime path.}
    \label{fig:ssde_overview}
\end{figure}

\subsection{Related Work}
Existing continuous-time SSDE inference methods differ in their model restrictions and treatment of uncertainty.
\citet{blackwell2003} considers SSDEs with regimes observed at discrete times.
\citet{blackwell2016} use uniformization for SDEs with state-dependent switching and tractable transition densities, jointly updating candidate times, states, and regimes using Metropolis--Hastings (MH).
In contrast, our sampler updates regimes conditional on the grid states using a Gibbs update.
\citet{hibbah2020} develop a data-augmentation MCMC method for nonlinear SDEs with state-dependent switching, but demonstrate only univariate models without observation error and note computational and mixing difficulties under dense augmentation.
\citet{kohs2021,kohs2022} demonstrate multivariate inference with observation error under more restrictive models: the former uses variational inference with state-independent diffusion, while the latter uses MCMC with linear-Gaussian dynamics.
\citet{stumpf-fetizon2026} provide discretization-free exact inference for nonlinear univariate SSDEs with state-dependent diffusion and direct observations, while multivariate extensions require substantially stronger conditions and are not demonstrated.

Our approach combines uniformization-based CTMC sampling~\citep{rao2013} with FNLE, a simulation-based inference method for Markovian time series~\citep{gloeckler2024}.
We use time-conditioned FNLE to approximate SDE transition densities over varying time intervals.
Similarly, neural stochastic flows learn time-conditioned SDE transition distributions~\citep{kiyohara2025} but are formulated without conditioning on SDE parameters.
In contrast, our transition-density approximation is conditioned on both time intervals and model parameters, allowing it to be incorporated into an approximate MCMC sampler that jointly infers latent states, regimes, and model parameters.
Table~\ref{tab:ssde_method_comparison} summarizes the differences between our method and existing SSDE inference approaches.

\begin{table*}[t]
    \setlength{\tabcolsep}{4pt}
    \caption{Applicability of SSDE Inference Methods}
    \label{tab:ssde_method_comparison}
    \begin{center}
    \small
    \begin{tabular*}{\textwidth}{@{\extracolsep{\fill}}lccccc@{}}
        \toprule
        Method
        & \shortstack{Observation\\error}
        & \shortstack{Nonlinear\\drift}
        & \shortstack{State-dependent\\diffusion}
        & \shortstack{Intractable\\transition density}
        & \shortstack{Multivariate\\states} \\
        \midrule
        \citet{blackwell2016}
        & \(\triangle\)
        & \(\triangle\)
        & \(\triangle\)
        & \tablecross
        & \tablecheck \\
        \citet{hibbah2020}
        & \(\triangle\)
        & \tablecheck
        & \tablecheck
        & \tablecheck
        & \(\triangle\) \\
        \citet{kohs2021}
        & \tablecheck
        & \(\triangle\)
        & \tablecross
        & \(\triangle\)
        & \tablecheck \\
        \citet{kohs2022}
        & \tablecheck
        & \tablecross
        & \tablecross
        & \tablecross
        & \tablecheck \\
        \citet{stumpf-fetizon2026}
        & \(\triangle\)
        & \tablecheck
        & \tablecheck
        & \tablecheck
        & \(\triangle\) \\
        Proposed method
        & \tablecheck
        & \tablecheck
        & \tablecheck
        & \tablecheck
        & \tablecheck \\
        \bottomrule
    \end{tabular*}
    \par\smallskip
    {\footnotesize
    \tablecheck: Demonstrated applicability;
    \(\triangle\): Applicability not established;
    \tablecross: Excluded by model assumptions.\par}
    \end{center}
\end{table*}

\section{Background}

\subsection{Switching Stochastic Differential Equations}

A switching stochastic differential equation (SSDE) combines a continuous latent state with a regime that determines its drift and diffusion functions~\citep{mao2006}.
Let \(Y_t \in \R^D\) be a continuous latent state and let \(W_t\in\R^D\) be a standard Brownian motion.
The regime \(Z_t\in\{1,\ldots,K\}\) evolves as a CTMC, where \(D\) is the state dimension and \(K\) is the number of regimes.
We consider an SSDE of the form
\begin{align}
    \dd Y_t
    = f(Y_t,Z_t)\dd t
    + g(Y_t,Z_t)\dd W_t,
    \label{eq:ssde}
\end{align}
where \(f:\R^D\times\{1,\ldots,K\}\to\R^D\) and \(g:\R^D\times\{1,\ldots,K\}\to\R^{D\times D}\) are the drift and diffusion functions, respectively.
For each regime \(k\), \(f(\cdot,k)\) and \(g(\cdot,k)\) define the corresponding regime-specific SDE.
The transition rate matrix \(Q=(q_{kh})\) of \(Z_t\) satisfies
\begin{align}
q_{kh} \geq 0 \quad (k\neq h),
\qquad
q_{kk} = -\sum_{h\neq k} q_{kh}.
\end{align}
On the observation interval \([0,T]\), we write \(Y=\{Y_t\}_{t\in[0,T]}\) and \(Z=\{Z_t\}_{t\in[0,T]}\) for the state and regime paths.
Their realizations are denoted by \(y=\{y_t\}_{t\in[0,T]}\) and \(z=\{z_t\}_{t\in[0,T]}\).

\subsection{Uniformization of Continuous-Time Markov Chains}
\label{sec:uniformization}

Uniformization represents a CTMC as a discrete-time Markov chain observed at event times of a Poisson process~\citep{rao2013}.
For a fixed observation grid, augmenting the regime path \(Z\) with candidate times \(\Tcand\) gives an equivalent discrete representation through the almost-sure one-to-one mapping
\begin{align}
    (\widetilde Z_{1:L},\Tevent)=\Psi(Z,\Tcand),
    \label{eq:uniformization_psi}
\end{align}
where \(\Tevent=\Ttrue\cup\Tcand\) combines the true jump times \(\Ttrue\), at which the regime changes, and candidate times \(\Tcand\).
For \(N\) observations at fixed times \(\Tobs=\{t_i^{\mathrm{obs}}\}_{i=1}^N\), with \(0=t_1^{\mathrm{obs}}<\cdots<t_N^{\mathrm{obs}}=T\), the combined grid is
\begin{align}
    \Tall=\Tobs\cup\Tevent=\{t_1,\ldots,t_L\},
    \nonumber\\
    0=t_1<\cdots<t_L=T.
    \label{eq:augmented_grid}
\end{align}
Here, \(\Tevent\) and \(\Tobs\) are almost surely disjoint, and \(\widetilde Z_j=Z_{t_j}\).
The sequence \(z_{1:L}=(z_{t_1},\ldots,z_{t_L})\) is its realization, distinct from the full path \(z\).
Under the right-continuous convention, \(z_j\) is the regime on \([t_j,t_{j+1})\) for \(j=1,\ldots,L-1\); \(z_1\) and \(z_L\) are the initial and terminal regimes.

As an illustrative model, consider conditionally independent latent states with densities \(p(\tilde y_j\mid z_j,\theta)\) and conditionally independent observations with densities \(p(x_i\mid\tilde y_{\ell(i)},\theta)\), where \(\theta\) is fixed and \(t_{\ell(i)}=t_i^{\mathrm{obs}}\).
Given \(\tilde y_{1:L}\) and \(\theta\), observation factors cancel from the regime conditional.
Writing \(e_j(k)=p(\tilde y_j\mid z_j=k,\theta)\) and \(\mathcal E=\{e_j(k)\}_{j=1,\ldots,L;\,k=1,\ldots,K}\), we obtain
\begin{align}
    &p(z_{1:L}\mid\Tall,\mathcal E,Q)
    \nonumber\\
    &\quad\propto
    p(z_1)e_1(z_1)
    \prod_{j=2}^{L}
    A_j(z_{j-1},z_j)e_j(z_j).
    \label{eq:uniformization_hmm}
\end{align}
The transition matrix \(A_j\) depends on \(Q\), the uniformization rate \(\Omega\), and whether \(t_j\) is an event or observation-only time (Appendix~\ref{app:uniformization_matrices}).
This hidden Markov model (HMM) factorization permits forward-filtering backward-sampling (FFBS)~\citep{fruhwirth2006}.

Conditional on the current path and \(Q\), candidate times are sampled as virtual-jump times from a Poisson process with rate \(\Omega(Q)+q_{Z_t Z_t}\), where \(\Omega(Q)>\max_k(-q_{kk})\) is deterministic given \(Q\) (Appendix~\ref{app:uniformization_matrices}).
After FFBS, candidate times may become true jump times, and we update
\begin{align}
    \Ttrue
    \leftarrow
    \{t_j\in\Tall:
    z_{j-1}\neq z_j,\ j=2,\ldots,L\}.
    \label{eq:t_true_update}
\end{align}

\subsection{Factorized Neural Likelihood Estimation for Transition Densities}
Factorized neural likelihood estimation (FNLE) is a method for approximating intractable transition densities in Markov process models~\citep{gloeckler2024}.

For a discrete-time Markov trajectory \(y_{1:J}\) of length \(J\), the trajectory likelihood factorizes as
\begin{align}
    p(y_{1:J} \mid \theta) = p(y_1) \prod_{j=1}^{J-1} p(y_{j+1}\mid y_j,\theta),
\end{align}
where \(\theta\) denotes the parameters that determine the transition density.

FNLE learns \(\phat(y_{j+1}\mid y_j,\theta)\approx p(y_{j+1}\mid y_j,\theta)\) by maximizing the conditional log-likelihood of simulated transitions, where \(\phi\) denotes the density-estimator parameters.
Conditional normalizing flows support both density evaluation and sampling~\citep{papamakarios2019,gloeckler2024}, enabling likelihood approximation and SIR proposals in our sampler.

\section{Problem setting}
\label{sec:problem_setting}
We define the model targeted by our proposed sampler.
We assume that the drift and diffusion functions share common functional forms across regimes as follows:
\begin{align}
    f(y,k) = f(y;\theta_k),
    \qquad
    g(y,k) = g(y;\theta_k),
\end{align}
where \(\theta_k\) is the SDE parameter vector in regime \(k\).
Note that our method is not limited to this case and that functional forms can depend on regimes.
We let \(\Theta\) collect the unconstrained SDE parameters, with separate entries for regime-specific parameters and a single entry for each shared parameter, and write
\begin{align}
    \theta_k=h_k(\Theta),
    \qquad k=1,\ldots,K.
\end{align}
Here, \(h_k\) selects the parameters used in regime \(k\) and applies the transformations required by their constraints.

At the discrete observation times in \(\Tobs\), observations are modeled as
\begin{align}
    X_i \mid Y_{t^{\mathrm{obs}}_i},\tau^2
    \sim
    \Normal\!\left(
    Y_{t^{\mathrm{obs}}_i},
    \diag(\tau^2)
    \right).
    \label{eq:obs_model}
\end{align}
Here, \(X_i\in\R^D\) denotes the random observation at time \(t_i^{\mathrm{obs}}\), and \(x_i\) denotes its realized value.
We write \(X=\{X_i\}_{i=1}^{N}\) and \(x=\{x_i\}_{i=1}^{N}\).
Furthermore, \(\tau^2\in\R_+^D\) is the vector of observation-noise variances.
The initial state and regime have independent priors \(p(y_1)\) and \(p(z_1)\) that do not depend on the unknown parameters.
Prior settings for unknown variables are provided in Appendix~\ref{app:prior_settings}.
The target posterior distribution is written as
\begin{align}
    &p(y, z, \Theta, \tau^2, Q \mid x) \propto p(x, y, z, \Theta, \tau^2, Q) \\
    &\quad= p(x \mid y, \tau^2)p(y \mid z, \Theta)p(z \mid Q) \nonumber\\
    &\qquad{}\times p(\Theta)p(\tau^2)p(Q).
    \label{eq:exact_target}
\end{align}

\section{Proposed Method}
\label{sec:proposed_method}

We first construct an exact MCMC sampler for tractable SSDEs, then use time-conditioned FNLE to obtain an approximate sampler for intractable SSDEs.
We call an SSDE tractable when exact transition-density evaluation and exact sampling from the corresponding bridge distribution are available for each regime-specific SDE.
We refer to all other SSDEs as intractable in this sense.
Here, a bridge distribution is the conditional distribution of the state path between two fixed times, given the endpoint states and the SDE parameters.
We consider a single series; Appendix~\ref{app:multiple_series} gives the extension to multiple series.
On the combined grid in Eq.~\eqref{eq:augmented_grid}, write \(y_j=y_{t_j}\) for \(j=1,\ldots,L\) and \(\Delta_j=t_j-t_{j-1}\) for \(j=2,\ldots,L\).
We write \(Y_{\Tall}\) and \(Y_{\setminus\Tall}\) for the restrictions of \(Y\) to the combined grid and the remaining times, respectively.

\subsection{Exact MCMC for Tractable SSDEs}
\label{sec:block_gibbs_updates}

Uniformization augments the posterior in Eq.~\eqref{eq:exact_target} with \(\Tcand\).
For fixed \(\Tevent\), consider successive block Gibbs updates of \(Y=(Y_{\Tall},Y_{\setminus\Tall})\), \((Z,Y_{\setminus\Tall})\), and \((\Theta,Y_{\setminus\Tall})\).
The unused draws of \(Y_{\setminus\Tall}\) can be trimmed by partially collapsed Gibbs sampling~\citep{vandyk2008}, leaving the following finite-dimensional updates.
Consequently, only the grid values \(y_{1:L}\) are explicitly sampled; the intervening paths are marginalized.

\paragraph{Candidate Times}
Given the current regime path and \(Q\), we sample \(\Tcand\) using the Poisson--Uniform construction in Section~\ref{sec:uniformization} and set \(\Tevent\) and \(\Tall\).
The rate \(\Omega\) and uniformized transition matrix \(B\) are computed as in Appendix~\ref{app:uniformization_matrices}.

\paragraph{Latent States at Newly Inserted Candidate Times}
Before updating all grid values, we sample each block \(y_{a:b}\) of state values at candidate times from the bridge distribution
\begin{align}
    &p(y_{a:b}\mid
    y_{a-1},y_{b+1},
    z_{1:L},\Delta_{a:b+1},\Theta),
    \label{eq:y_bridge_conditional}
\end{align}
where \(y_{a:b}\) is a consecutive block of \(\{y_t\}_{t\in\Tcand}\), with fixed neighboring values \(y_{a-1}\) and \(y_{b+1}\) at adjacent points of \(\Tobs\cup\Ttrue\).

For SSDEs, we redefine \(e_j\) as the exact transition density for \(j=2,\ldots,L\):
\begin{align}
    e_j(k)=p(y_j\mid y_{j-1},\theta_k,\Delta_j).
    \label{eq:ej_def}
\end{align}
Because \(z_{j-1}\) is the regime on \([t_{j-1},t_j)\), the transition to \(y_j\) contributes \(e_j(z_{j-1})\).

\paragraph{Latent States on the Combined Time Grid}
We update \(y_{1:L}\) using an MH kernel targeting
\begin{align}
    &p(y_{1:L}\mid
    z_{1:L},\Tall,\Theta,\tau^2,x)
    \nonumber\\
    &\quad \propto p(y_1) \prod_{j=2}^{L}
    e_j(z_{j-1}) \nonumber\\
    &\qquad{}\times \prod_{i=1}^{N}
    p(
    x_i\mid y_{t_i^{\mathrm{obs}}},\tau^2
    ).
    \label{eq:y_conditional}
\end{align}
\paragraph{Regime Path}
Marginalizing \(Y_{\setminus\Tall}\) gives
\begin{align}
    &p(z_{1:L}\mid y_{1:L},\Tevent,\Theta,Q,\tau^2,x)
    \nonumber\\
    &\quad\propto p(z_1)\prod_{j=2}^{L}
    A_j(z_{j-1},z_j)e_j(z_{j-1}).
    \label{eq:exact_regime_conditional}
\end{align}
This HMM factorization permits FFBS; Appendix~\ref{app:proof_regime_conditional} gives the derivation.

\paragraph{SDE Parameters}
The conditional density is
\begin{align}
    p(\Theta\mid y_{1:L},z_{1:L},\Tall)
    &\propto p(\Theta)\prod_{j=2}^{L}
    e_j(z_{j-1}).
    \label{eq:theta_conditional}
\end{align}
We update \(\Theta\) using an MH kernel targeting this conditional.

\paragraph{Observation Noise and Transition Rates}
We sample \(\tau^2\) and the off-diagonal entries of \(Q\) from their conjugate full conditionals, with virtual jumps marginalized in the \(Q\) update; see Appendix~\ref{app:conjugate_updates}.

With exact sampling from the bridge distribution and conditional-invariant MH kernels, this sweep preserves the exact posterior marginal of the regime path, latent values at \(\Tobs\cup\Ttrue\), and parameters~\citep{liu2001}.
Appendix~\ref{app:exact_block_gibbs} gives the full derivation.

\subsection{Time-Conditioned FNLE for SDE Transition Densities}

To construct the approximate MCMC sampler, we approximate the SDE transition densities required by the exact updates.
For a regime-specific SDE, the transition density is
\begin{align}
    p(y^{\mathrm{next}}\mid
    y^{\mathrm{prev}},\theta_k,\Delta).
\end{align}
Here, \(y^{\mathrm{prev}}\) and \(y^{\mathrm{next}}\) are state values separated by elapsed time \(\Delta\).
We train an FNLE model with the elapsed time \(\Delta\) as a conditioning variable to obtain
\begin{align}
    \phat(y^{\mathrm{next}}\mid
    y^{\mathrm{prev}},\theta_k,\Delta)
    \approx
    p(y^{\mathrm{next}}\mid
    y^{\mathrm{prev}},\theta_k,\Delta).
\end{align}
A single FNLE model can be shared across regimes with common drift and diffusion forms by conditioning on \(\theta_k\); otherwise, a separate model can be trained for each distinct pair of functional forms.
The model and training settings are given in Appendix~\ref{app:fnle_training}.

\subsection{Approximate Markov Chain Monte Carlo Sampler with FNLE}

We obtain the approximate sampler by replacing the transition densities in Section~\ref{sec:block_gibbs_updates} with FNLE densities and using SIR to sample state values at candidate times.
For \(j=2,\ldots,L\), define
\begin{align}
    \hat e_j(k) = \phat(y_j\mid y_{j-1},\theta_k,\Delta_j).
    \label{eq:ehat_def}
\end{align}
Replacing \(e_j\) with \(\hat e_j\) in Eqs.~\eqref{eq:y_conditional}, \eqref{eq:exact_regime_conditional}, and~\eqref{eq:theta_conditional} defines the approximate conditionals on a fixed grid.
The chain factorization still permits FFBS for \(z_{1:L}\).
For \(y_{1:L}\) and \(\Theta\), we use the No-U-Turn Sampler (NUTS)~\citep{hoffman2014}, using gradients of the corresponding FNLE-based conditional log densities.
The candidate-time, observation-noise, and transition-rate updates are unchanged.

For the state values at candidate times, SIR~\citep{rubin1987} approximates the bridge distribution in Eq.~\eqref{eq:y_bridge_conditional} with FNLE transition densities.
Particles are sampled forward from the fixed left endpoint and weighted by the FNLE density at the fixed right endpoint (Appendix~\ref{app:sir_candidate_states}).
Thus, FNLE must support conditional sampling and differentiable density evaluation.

\paragraph{Algorithm Summary}
Algorithm~\ref{alg:ssde_sampler} summarizes the proposed sampler with FNLE.
Here, \(N_{\mathrm{iter}}\) denotes the total number of MCMC sweeps.
Appendix~\ref{app:computational_complexity} describes its per-sweep computational costs.

\begin{algorithm}[t]
\caption{Continuous-time SSDE sampler with time-conditioned FNLE}
\label{alg:ssde_sampler}
\small
\algrenewcommand\algorithmicindent{0.8em}
\begin{algorithmic}[1]
\Require Data \(\{(x_i,t_i^{\mathrm{obs}})\}_{i=1}^{N}\)
\Statex Initial \(\Ttrue,Q,\Theta,\tau^2\); trained \(\phat\)
\State Set \(\Tall=\Tobs\cup\Ttrue\)
\State Initialize \(y_{1:L}\) and \(z_{1:L}\) on \(\Tall\), consistent with \(\Ttrue\)
\State Compute \(\Omega,B\) (Eq.~\eqref{eq:B_uniformization})
\For{iteration \(n=1,\ldots,N_{\mathrm{iter}}\)}
    \State Sample \(\Tcand\) (Section~\ref{sec:uniformization})
    \State Set \(\Tevent =\Ttrue\cup\Tcand\)
    \State Set \(\Tall\) (Eq.~\eqref{eq:augmented_grid})
    \State Re-evaluate \(z_{1:L}\) on the new \(\Tall\)
    \State SIR: sample \(\{y_t\}_{t\in\Tcand}\)
    \State NUTS: sample \(y_{1:L}\) (Eq.~\eqref{eq:y_conditional})
    \State FFBS: sample \(z_{1:L}\) (Eq.~\eqref{eq:exact_regime_conditional})
    \State Set \(\Ttrue\) (Eq.~\eqref{eq:t_true_update})
    \State NUTS: sample \(\Theta\) (Eq.~\eqref{eq:theta_conditional})
    \State Sample \(\tau^2\) (Eq.~\eqref{eq:tau_conditional})
    \State Sample \(\{q_{kh}\}_{k\neq h}\) (Eq.~\eqref{eq:Q_conditional})
    \State Set \(q_{kk}\) (Eq.~\eqref{eq:Q_diagonal}); recompute \(\Omega,B\)
\EndFor
\end{algorithmic}
\end{algorithm}

\subsection{Approximation Properties}

In our approximate MCMC sampler, FNLE approximates the transition densities, and the finite number of SIR particles introduces a separate error in sampling from the bridge distribution.
Because FNLE densities do not necessarily satisfy the Chapman--Kolmogorov equation, inserting intermediate grid points can change the implied transition density.
Consequently, successive sweeps apply the FNLE-based state, regime, and parameter updates on different grids and are not guaranteed to preserve a common approximate target.
Exact or Chapman--Kolmogorov-consistent transition densities remove this inconsistency, but not the error due to the finite number of SIR particles.
On a fixed grid, an MH transition, such as NUTS with fixed tuning parameters, preserves its conditional target and introduces no additional target approximation~\citep{liu2001}.

We assess FNLE error through the displacement of conditional parameter modes.
Consider one series with a realized combined time grid \(\mathcal G=\Tall=\{0=t_1<\cdots<t_L=T\}\), and condition on fixed \(y_{1:L}\) and \(z_{1:L}\).
Let \(q_{\mathcal G}\) and \(\hat q_{\mathcal G}\) be the exact and FNLE-based versions of the unnormalized conditional density in Eq.~\eqref{eq:theta_conditional}.
Let \(\mathcal K\) contain the modes of both \(q_{\mathcal G}\) and \(\hat q_{\mathcal G}\).
For \(j=2,\ldots,L\), define the local log-density error as
\begin{align}
    \eta_j(\Theta)
    &=\log\phat(y_j\mid y_{j-1},h_{z_{j-1}}(\Theta),\Delta_j)
    \nonumber\\
    &\quad-\log p(y_j\mid y_{j-1},h_{z_{j-1}}(\Theta),\Delta_j).
    \label{eq:prop2_local_error}
\end{align}
Write \(\ell_{\mathcal G}=T^{-1}\log q_{\mathcal G}\) and \(\hat\ell_{\mathcal G}=T^{-1}\log\hat q_{\mathcal G}\), and let \(\mathcal M_{\mathcal G}=\operatorname*{arg\,max}_{\Theta\in\mathcal K}\ell_{\mathcal G}(\Theta)\).
For a function \(e\), let \(\operatorname{osc}_{\mathcal K}(e)=\sup_{\mathcal K}e-\inf_{\mathcal K}e\), and let \(\operatorname{dist}(\Theta,\mathcal M_{\mathcal G})=\inf_{\vartheta\in\mathcal M_{\mathcal G}}\|\Theta-\vartheta\|_2\).
We use oscillation because additive log-density errors independent of \(\Theta\) do not affect the modes.
The following proposition bounds the displacement of conditional parameter modes caused by FNLE approximation.

\begin{samepage}
\begin{proposition}[Stability of FNLE-based conditional parameter modes]
\label{prop:fnle_conditional_mode_stability}
Assume that both \(q_{\mathcal G}\) and \(\hat q_{\mathcal G}\) attain their maxima on \(\mathcal K\) for every \(\mathcal G\) and its associated fixed \(y_{1:L}\) and \(z_{1:L}\).
Choose any \(\Theta_{\mathcal G}^{\star}\in\mathcal M_{\mathcal G}\).
Suppose that there exist \(\delta\geq0\) and \(c>0\) such that, uniformly over \(\mathcal G\) and its associated fixed \(y_{1:L}\) and \(z_{1:L}\),
\begin{align}
    \operatorname{osc}_{\mathcal K}(\eta_j)
    &\leq\delta,
    \quad j=2,\ldots,L,
    \nonumber\\
    \ell_{\mathcal G}(\Theta_{\mathcal G}^{\star})
    -\ell_{\mathcal G}(\Theta)
    &\geq
    \frac{c}{2}
    \operatorname{dist}(\Theta,\mathcal M_{\mathcal G})^2,
    \quad \Theta\in\mathcal K.
    \label{eq:prop2_conditional_conditions}
\end{align}
For any \(\hat\Theta_{\mathcal G}\in\operatorname*{arg\,max}_{\Theta\in\mathcal K}\hat\ell_{\mathcal G}(\Theta)\), set \(D_{\mathcal G}=\operatorname{dist}(\hat\Theta_{\mathcal G},\mathcal M_{\mathcal G})\).
Let \(\E_{\mathrm{unif}}[\cdot\mid Q]\) denote expectation over the uniformization event grid generated by a Poisson process with constant rate \(\Omega(Q)\) on \([0,T]\), conditional on \(Q\), with the observation grid fixed.
Then
\begin{align}
    \E_{\mathrm{unif}}[D_{\mathcal G}\mid Q]
    &\leq
    \left[
    \frac{2\delta}{c}
    \left\{\frac{N-1}{T}+\Omega(Q)\right\}
    \right]^{1/2}.
    \label{eq:prop2_expected_mode_stability}
\end{align}
\end{proposition}
\end{samepage}

The first condition bounds parameter-dependent variation in the FNLE log-density error; the second is a quadratic separation condition requiring the normalized log posterior to decrease at least quadratically with distance from its mode set.
Both conditions must hold uniformly over grids and fixed latent variables and are not guaranteed by FNLE training.
This local sensitivity result does not guarantee the marginal posterior or the full sampler.
At a fixed observation frequency \((N-1)/T\), the bound remains uniformly bounded as \(N\) increases, provided that \(\Omega(Q)\), \(\delta\), and \(c^{-1}\) remain bounded.
A proof is provided in Appendix~\ref{app:proof_prop2}.

\section{Numerical Experiments}

MCMC settings and initialization are detailed in Appendix~\ref{app:mcmc_settings}.
Computing infrastructure and software are described in Appendices~\ref{app:computing_infrastructure} and~\ref{app:software_assets}, respectively.

\begin{figure}[t!]
    \centering
    \includegraphics[width=\linewidth]{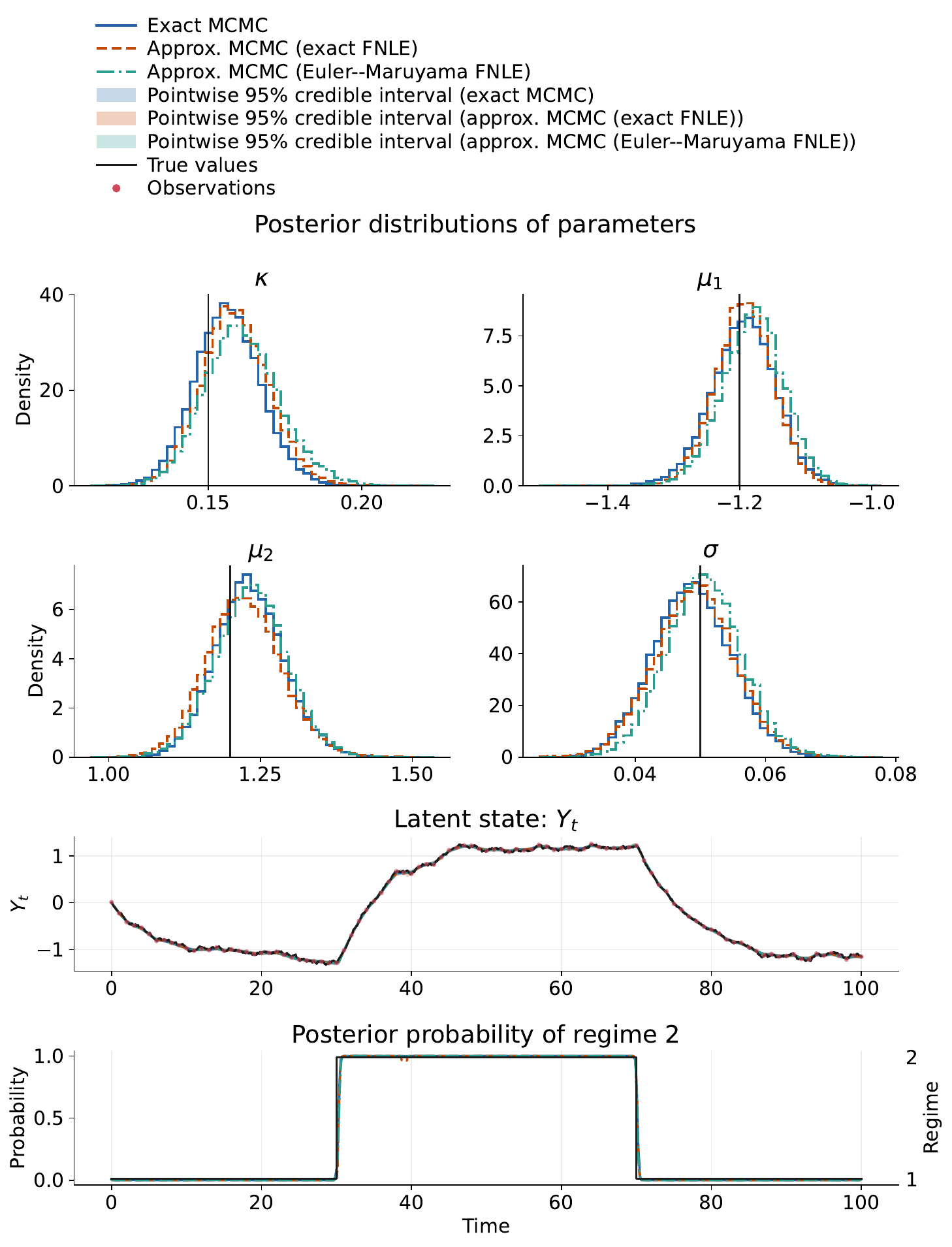}
    \caption{Comparison of inference results for the OU experiment using an exact MCMC sampler and two approximate MCMC samplers with FNLE trained on exact transitions and Euler--Maruyama simulations, respectively.}
    \label{fig:ou_posterior_comparison}
\end{figure}

\subsection{Posterior Comparison on an Ornstein--Uhlenbeck Model}
\label{sec:ou_comparison}

\begin{table*}[t!]
    \setlength{\tabcolsep}{4pt}
    \caption{Characteristics of the SSDE Models Used to Generate the Synthetic Datasets}
    \label{tab:synthetic_model_characteristics}
    \begin{center}
    \small
    \begin{tabular*}{\textwidth}{@{\extracolsep{\fill}}lccccc@{}}
        \toprule
        Model
        & \begin{tabular}[c]{@{}c@{}}Observation\\error\end{tabular}
        & \begin{tabular}[c]{@{}c@{}}Nonlinear\\drift\end{tabular}
        & \begin{tabular}[c]{@{}c@{}}State-dependent\\diffusion\end{tabular}
        & \begin{tabular}[c]{@{}c@{}}Intractable\\transition density\end{tabular}
        & \begin{tabular}[c]{@{}c@{}}Multivariate\\states\end{tabular} \\
        \midrule
        Ornstein--Uhlenbeck
        & \tablecheck
        & \tablecross
        & \tablecross
        & \tablecross
        & \tablecross \\
        Lotka--Volterra
        & \tablecheck
        & \tablecheck
        & \tablecheck
        & \tablecheck
        & \tablecheck \\
        Gene-expression CLE
        & \tablecheck
        & \tablecross
        & \tablecheck
        & \tablecheck
        & \tablecheck \\
        \begin{tabular}[c]{@{}l@{}}Susceptible--infected--recovered\\epidemic model\end{tabular}
        & \tablecheck
        & \tablecheck
        & \tablecheck
        & \tablecheck
        & \tablecheck \\
        \bottomrule
    \end{tabular*}
    \par\smallskip
    {\footnotesize
    \tablecheck: Present;
    \tablecross: Absent.\par}
    \end{center}
\end{table*}

\begin{figure*}[t!]
    \centering
    \includegraphics[width=\textwidth]{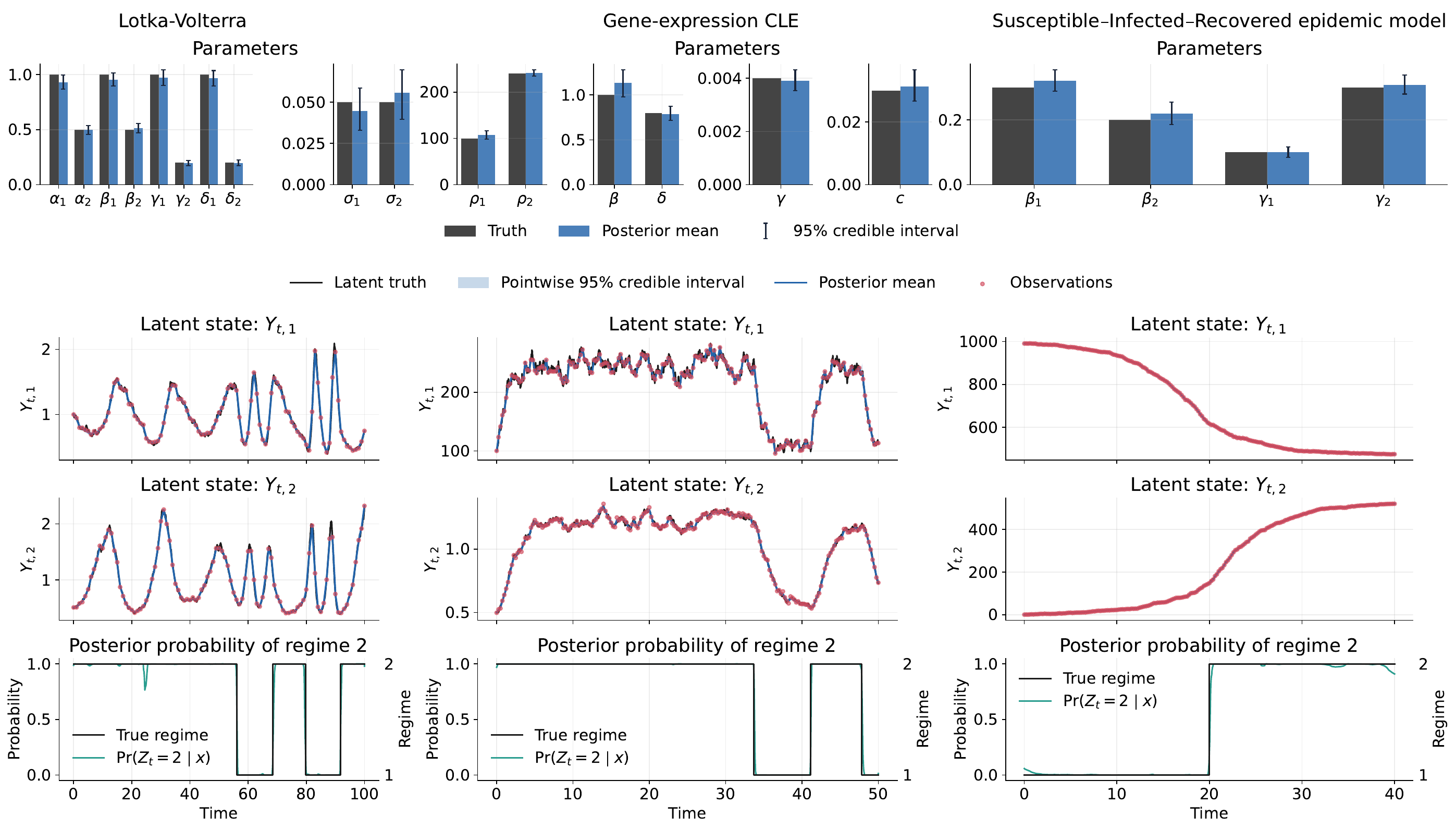}
    \caption{Inference results for synthetic data from three SSDEs with regime-specific Lotka--Volterra, gene-expression chemical Langevin, and susceptible--infected--recovered epidemic models (left to right).}
    \label{fig:synthetic_inference_results}
\end{figure*}

To assess the effect of FNLE approximation error, we compare posterior distributions obtained with two approximate MCMC samplers and an exact MCMC sampler (Section~\ref{sec:block_gibbs_updates}) for a tractable SSDE.
We use an SSDE with \(K=2\) and an Ornstein--Uhlenbeck (OU) model~\citep{uhlenbeck1930} in each regime.
In regime \(k\), the model is
\begin{align}
    \dd Y_t=\kappa(\mu_k-Y_t)\dd t+\sigma\dd W_t.
    \label{eq:synthetic_ou_model}
\end{align}
Here, \(W_t\) is a one-dimensional standard Brownian motion.
Each regime-specific OU SDE has an exact transition density and bridge distribution, making this SSDE tractable under the definition in Section~\ref{sec:proposed_method}.
Figure~\ref{fig:ou_posterior_comparison} compares posterior parameter distributions from the three samplers.
The sampler using FNLE trained on exact transitions produces posterior marginals close to those from the exact sampler.
The marginals remain similar when FNLE is trained on Euler--Maruyama simulations instead, as in the subsequent experiments on intractable SSDEs.
All three samplers recovered the SDE parameters and regime path, with \(\hat R\) values close to one (Appendix Table~\ref{tab:ou_diagnostics}).
The experimental settings are given in Appendix~\ref{app:synthetic_ou}.

\subsection{Synthetic-Data Experiments}
\label{sec:synthetic_data_inference}

We evaluate regime-path recovery and estimation of \(\Theta\) using synthetic data from three SSDE models with \(K=2\) and known regimes and parameters.
The remaining three rows of Table~\ref{tab:synthetic_model_characteristics} summarize these intractable SSDEs; the first row describes the tractable OU benchmark in Section~\ref{sec:ou_comparison}.

All three models include observation error, state-dependent diffusion, and multivariate latent states.
The previous methods in Table~\ref{tab:ssde_method_comparison} have not demonstrated applications combining these features.

\(W_{t,1}\) and \(W_{t,2}\) denote independent Brownian motions.
Complete settings of these three experiments are given in Appendices~\ref{app:synthetic_lv}, \ref{app:synthetic_cle}, and \ref{app:synthetic_sir}, respectively.

\paragraph{Lotka--Volterra Equation}
We write \(Y_{t,1}\) and \(Y_{t,2}\) for predator and prey abundance, respectively.
The drift follows the classical Lotka--Volterra (LV) predator--prey model~\citep{lotka1920,volterra1927}, and independent multiplicative noise is introduced as in \citet{li2009}.
In regime \(k\), the model is
\begin{equation}
\begin{aligned}
    \dd Y_{t,1}
    &=\left(\gamma_kY_{t,2}Y_{t,1}
    -\delta_kY_{t,1}\right)\dd t \\
    &\quad+\sigma_1Y_{t,1}\,\dd W_{t,1},\\
    \dd Y_{t,2}
    &=\left(\alpha_kY_{t,2}
    -\beta_kY_{t,2}Y_{t,1}\right)\dd t \\
    &\quad+\sigma_2Y_{t,2}\,\dd W_{t,2}.
\end{aligned}
\label{eq:synthetic_lv_model}
\end{equation}

\paragraph{Chemical Langevin Equation}
We write \(Y_{t,1}\) for the mRNA copy number and \(Y_{t,2}\) for the fluorescence intensity proportional to the protein copy number, using continuous approximations to the copy numbers.
We use a chemical Langevin equation (CLE)~\citep{gillespie2000}, whose dynamics in regime \(k\) are
\begin{equation}
\begin{aligned}
    \dd Y_{t,1}
    &=(\alpha_k-\beta Y_{t,1})\dd t \\
    &\quad+\sqrt{\alpha_k+\beta Y_{t,1}}\,\dd W_{t,1},\\
    \dd Y_{t,2}
    &=(\gamma Y_{t,1}-\delta Y_{t,2})\dd t \\
    &\quad+c\sqrt{\gamma Y_{t,1}+\delta Y_{t,2}}\,\dd W_{t,2}.
\end{aligned}
\label{eq:synthetic_cle_model}
\end{equation}
For inference, we reparameterize the CLE using \(\rho_k=\alpha_k/\beta\) and \(\beta\) instead of \(\alpha_k\) and \(\beta\) to improve MCMC mixing.
The coefficient \(c^2\) converts protein copy number to fluorescence intensity.

\paragraph{Susceptible--Infected--Recovered Epidemic Model}
We write \(Y_{t,1}\) and \(Y_{t,2}\) for the susceptible and recovered populations, respectively.
The drift follows the classical susceptible--infected--recovered epidemic model~\citep{kermack1927}, and the diffusion terms follow the chemical Langevin approximation~\citep{gillespie2000}.
The infected population is determined by \(I_t=N_{\mathrm{pop}}-Y_{t,1}-Y_{t,2}\), where \(N_{\mathrm{pop}}\) is the fixed total population.
In regime \(k\), the model is
\begin{equation}
\begin{aligned}
    \dd Y_{t,1}
    &=-\frac{\beta_kY_{t,1}I_t}{N_{\mathrm{pop}}}\dd t -\sqrt{\frac{\beta_kY_{t,1}I_t}{N_{\mathrm{pop}}}}\,\dd W_{t,1},\\
    \dd Y_{t,2}
    &=\gamma_kI_t\dd t
    +\sqrt{\gamma_kI_t}\,\dd W_{t,2}.
\end{aligned}
\label{eq:synthetic_sir_model}
\end{equation}

Posterior means were close to the true parameters with narrow 95\% credible intervals in all three experiments (Figure~\ref{fig:synthetic_inference_results}).
The posterior regime probabilities also recovered the true switching pattern.
The \(\tau_1\) parameter in the CLE experiment and \(\tau_1,\tau_2\) in the susceptible--infected--recovered experiment showed poor mixing (Appendix Tables~\ref{tab:cle_diagnostics} and~\ref{tab:sir_diagnostics}); diagnostics for the LV experiment are in Appendix Table~\ref{tab:lv_diagnostics}.

\subsection{Real-Data Application}

We applied the proposed method to 71 paired abundance observations of \textit{Didinium} and \textit{Paramecium} from \citet{veilleux1979}, using the LV SSDE in Eq.~\eqref{eq:synthetic_lv_model} with \(K=2\).
Data preprocessing and inference settings are given in Appendix~\ref{app:real_data_settings}.

\begin{figure}[t!]
    \centering
    \includegraphics[width=\linewidth]{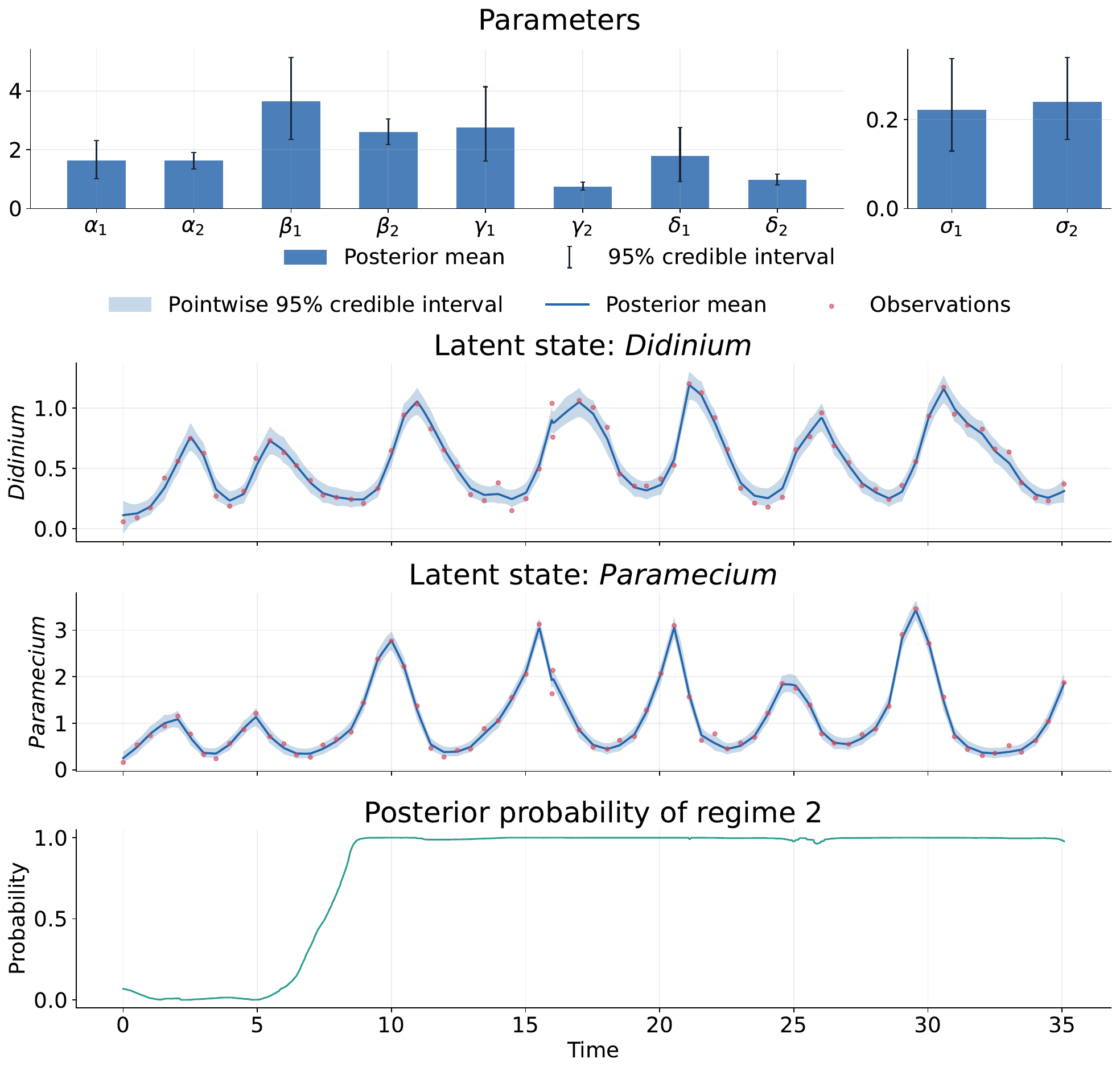}
    \caption{Real-data inference results for the \textit{Didinium}--\textit{Paramecium} time series using a Lotka--Volterra SSDE.}
    \label{fig:real_data_inference_results}
\end{figure}

As shown in Figure~\ref{fig:real_data_inference_results}, the posterior probability of regime 2 increased around \(t=7\) and remained close to one thereafter.
Notably, the initial regime identified in an unsupervised manner closely corresponds to the transient period that was manually excluded prior to fitting a LV model in a previous analysis of these data~\citep{liao2020}.
Appendix Figure~\ref{fig:real_data_multichain_trace} shows similar chain-specific marginal posterior distributions and overlapping traces across the four chains.
The diagnostics in Appendix Table~\ref{tab:real_lv_diagnostics} also suggest convergence.

\section{Conclusion}

We proposed a broadly applicable approximate MCMC sampler combining time-conditioned FNLE and uniformization for joint inference of SSDE latent states, regimes, and parameters.
Starting from an exact uniformization-based MCMC sampler for tractable SSDEs, we used FNLE to approximate transition densities and SIR to sample approximately from bridge distributions, extending the approach to intractable SSDEs.
In the OU benchmark, approximate samplers using FNLE trained on exact or Euler--Maruyama transitions produced posterior parameter marginals similar to those from exact MCMC.
The approximate sampler also recovered parameters and regime paths for the three SSDE models, to which existing methods have limited applicability.
In the real-data analysis of the \textit{Didinium}--\textit{Paramecium} time series, the method identified a persistent regime transition from the initial transient dynamics to the subsequent predator--prey dynamics.

\clearpage
\subsection*{AI use statement}

In this work, we used generative AI tools to help develop theoretical models or conceptual frameworks, formulate mathematical claims, provide critical ingredients for proving mathematical claims, assist in the writing of proofs, design or provide feedback on research methodology or experiments, implement methods, and assist with translation.
We have not used generative AI tools to generate synthetic datasets, propose or refine hypotheses, clean and reformat datasets, support qualitative and thematic data analysis, or interpret results.
Additionally, we used generative AI tools to create or modify scientific figures or images, suggest experimental parameters, create or edit software code, draft parts of a research paper, summarize or analyze existing literature, brainstorm, source or search for information, edit a research paper to improve readability, identify relevant literature, and propose a title or keywords.
We have reviewed all AI-assisted work.
We manually reviewed all LLM-generated code, draft text, and proofs. All changes, whether AI-assisted or manual, were tracked using \texttt{git}, and we inspected the corresponding diffs to carefully review modifications made with the assistance of LLMs. LLM-assisted proofs were checked by the authors for correctness. All references were added to the BibTeX file manually after verifying their bibliographic information against the original sources and reading the cited works to confirm their content and relevance.
We take responsibility for the final content of this work, including text, claims, or artifacts produced with the aid of generative AI.

\ifpreprintversion
\subsection*{Code Availability}

The code for the numerical experiments is available at \url{https://github.com/shion-h/fnle-ssde}.
\fi

\bibliographystyle{plainnat}
\bibliography{references}

\ifpreprintversion\else
\section*{Checklist}

\begin{enumerate}

  \item For all models and algorithms presented, check if you include:
  \begin{enumerate}
    \item A clear description of the mathematical setting, assumptions, algorithm, and/or model. Yes; Sections~\ref{sec:problem_setting} and~\ref{sec:proposed_method}, and Algorithm~\ref{alg:ssde_sampler}.
    \item An analysis of the properties and complexity (time, space, sample size) of any algorithm. Yes; Proposition~\ref{prop:fnle_conditional_mode_stability} and Appendices~\ref{app:exact_block_gibbs} and~\ref{app:computational_complexity}.
    \item (Optional) Anonymized source code, with specification of all dependencies, including external libraries. Yes; anonymized supplementary code includes \texttt{uv.lock} specifying dependencies and their versions.
  \end{enumerate}

  \item For any theoretical claim, check if you include:
  \begin{enumerate}
    \item Statements of the full set of assumptions of all theoretical results. Yes; Proposition~\ref{prop:fnle_conditional_mode_stability}.
    \item Complete proofs of all theoretical results. Yes; Appendices~\ref{app:exact_block_gibbs} and~\ref{app:proof_prop2}.
    \item Clear explanations of any assumptions. Yes; the discussion following Proposition~\ref{prop:fnle_conditional_mode_stability}.
  \end{enumerate}

  \item For all figures and tables that present empirical results, check if you include:
  \begin{enumerate}
    \item The code, data, and instructions needed to reproduce the main experimental results (either in the supplemental material or as a URL). Yes; supplementary experimental code, data retrieval code, and execution instructions.
    \item All the training details (e.g., data splits, hyperparameters, how they were chosen). Yes; Appendix~\ref{app:fnle_training}.
    \item A clear definition of the specific measure or statistics and error bars (e.g., with respect to the random seed after running experiments multiple times). Yes; legends in Figures~\ref{fig:ou_posterior_comparison}--\ref{fig:real_data_inference_results}.
    \item A description of the computing infrastructure used. (e.g., type of GPUs, internal cluster, or cloud provider). Yes; Appendix~\ref{app:computing_infrastructure}.
  \end{enumerate}

  \item If you are using existing assets (e.g., code, data, models) or curating/releasing new assets, check if you include:
  \begin{enumerate}
    \item Citations of the creator If your work uses existing assets. Yes; Appendix~\ref{app:software_assets}.
    \item The license information of the assets, if applicable. Yes; Appendix~\ref{app:software_assets}.
    \item New assets either in the supplemental material or as a URL, if applicable. Yes; supplementary experimental code.
    \item Information about consent from data providers/curators. Not Applicable; we use publicly available ecological data.
    \item Discussion of sensible content if applicable, e.g., personally identifiable information or offensive content. Not Applicable; no personal or sensitive data are used.
  \end{enumerate}

  \item If you used crowdsourcing or conducted research with human subjects, check if you include:
  \begin{enumerate}
    \item The full text of instructions given to participants and screenshots. Not Applicable; no human participants.
    \item Descriptions of potential participant risks, with links to Institutional Review Board (IRB) approvals if applicable. Not Applicable; no human participants.
    \item The estimated hourly wage paid to participants and the total amount spent on participant compensation. Not Applicable; no human participants.
  \end{enumerate}

\end{enumerate}
\fi

\clearpage
\appendix
\onecolumn
\thispagestyle{empty}
\aistatstitle{\papertitle: \\ Supplementary Materials}

\section{Exact Sampling Construction}

We derive an exact version of Algorithm~\ref{alg:ssde_sampler} from the proper full-path posterior in Eq.~\eqref{eq:exact_target}, assuming exact SDE transition densities and sampling from bridge distributions are available.
For any time set \(\mathcal G\), define \(Y_{\mathcal G}=\{Y_t\}_{t\in\mathcal G}\) and \(Y_{\setminus\mathcal G}=\{Y_t\}_{t\in[0,T]\setminus\mathcal G}\), with realizations denoted by lowercase symbols, and let \(\mathcal T_{\mathrm{ret}}=\Tobs\cup\Ttrue\).
Each sweep first samples \((\Tcand,Y_{\setminus\mathcal T_{\mathrm{ret}}})\) conditional on the current regime, retained state values, and parameters.
With the resulting \(\Tevent\) and \(\Tall\) fixed, it updates the blocks \(Y\), \((Z,Y_{\setminus\Tall})\), and \((\Theta,Y_{\setminus\Tall})\), followed by \(\tau^2\).
The candidate-time augmentation is then marginalized before updating \(Q\).
The factorizations below permit trimming unused path draws by partially collapsed Gibbs sampling~\citep{vandyk2008}.

\subsection{Uniformization Matrices and Candidate Times}
\label{app:uniformization_matrices}

For a transition rate matrix \(Q\), choose a deterministic rate \(\Omega=\Omega(Q)>\max_k(-q_{kk})\).
The uniformized chain has transition matrix~\citep{rao2013}
\begin{align}
    B=I_K+\frac{Q}{\Omega},
    \label{eq:B_uniformization}
\end{align}
where \(I_K\) is the \(K\times K\) identity matrix.
On \(\Tall\), event times permit regime transitions, whereas observation-only times preserve the regime:
\begin{align}
    A_j=
    \begin{cases}
        B, & t_j\in\Tevent,\\
        I_K, & t_j\in\Tobs,
    \end{cases}
    \qquad j=2,\ldots,L.
    \label{eq:Aj_def}
\end{align}
Event times almost surely do not coincide with the fixed observation times.
See \citet{rao2013} for the uniformization construction and its MCMC application.

For candidate generation, let \([a_r,b_r]\), \(r=1,\ldots,R_Z\), be the intervals between consecutive times in \(\{0,T\}\cup\Ttrue\), with regime \(k_r\) in each interval's interior.
Conditional on the current path and \(Q\), sample independently across intervals
\begin{align}
    M_r \sim \Poisson\{(\Omega+q_{k_r k_r})(b_r-a_r)\},
    \label{eq:candidate_count}
\end{align}
followed by
\begin{align}
    u_{r,m}\mid M_r
    &\overset{\mathrm{i.i.d.}}{\sim}
    \Uniform(a_r,b_r),\qquad m=1,\ldots,M_r,
    \label{eq:candidate_times}
\end{align}
where \(\Poisson(\lambda)\) has mean \(\lambda\), and \(\Uniform(a,b)\) is uniform on \((a,b)\).
The sampled times \(u_{r,m}\) form \(\Tcand\).

\subsection{Block Gibbs Construction}
\label{app:exact_block_gibbs}

\subsubsection{Candidate Grids and Path Augmentation}

Conditional on \(Z\), \(\mathcal T_{\mathrm{ret}}\) is fixed, and the \((\Tcand,Y_{\setminus\mathcal T_{\mathrm{ret}}})\) update factors by conditional independence:
\begin{align}
    &p(\Tcand,y_{\setminus\mathcal T_{\mathrm{ret}}}\mid
    y_{\mathcal T_{\mathrm{ret}}},z,\Theta,Q,\tau^2,x)
    \nonumber\\
    &\quad=
    p(\Tcand\mid z,Q)
    \nonumber\\
    &\qquad{}\times
    p(y_{\setminus\mathcal T_{\mathrm{ret}}}\mid
    y_{\mathcal T_{\mathrm{ret}}},z,\Theta).
    \label{eq:exact_grid_augmentation}
\end{align}
The first factor is the conditional Poisson process in Eqs.~\eqref{eq:candidate_count} and~\eqref{eq:candidate_times}~\citep{rao2013}.

Use the local indices \(a:b\) from Eq.~\eqref{eq:y_bridge_conditional}, with endpoints \(t_{a-1}\) and \(t_{b+1}\) consecutive in \(\mathcal T_{\mathrm{ret}}\) and regime \(k\) between them.
By the Markov property and the Chapman--Kolmogorov equation, its candidate-time marginal is
\begin{align}
    &p(y_{a:b}\mid y_{a-1},y_{b+1},z,\Tall,\Theta)
    \nonumber\\
    &\quad=
    \frac{
        \prod_{j=a}^{b+1}
        p(y_j\mid y_{j-1},\theta_k,\Delta_j)
    }{
        p(y_{b+1}\mid y_{a-1},\theta_k,t_{b+1}-t_{a-1})
    }.
    \label{eq:exact_candidate_bridge}
\end{align}
Sampling the intervening paths from their bridge distributions conditional on \(y_{\Tcand}\) yields \(Y_{\setminus\mathcal T_{\mathrm{ret}}}\).

\subsubsection{Latent-State Update}

The \(Y\) block factors as
\begin{align}
    &p(y_{\Tall},y_{\setminus\Tall}\mid
    z,\Tevent,\Theta,\tau^2,x)
    \nonumber\\
    &\quad=
    p(y_{\Tall}\mid z,\Tevent,\Theta,\tau^2,x)
    \nonumber\\
    &\qquad{}\times
    p(y_{\setminus\Tall}\mid y_{\Tall},z,\Tevent,\Theta).
    \label{eq:exact_y_block}
\end{align}
Since \(z_{j-1}\) is the regime on \([t_{j-1},t_j)\), the Markov property gives
\begin{align}
    &p(y_{1:L}\mid z_{1:L},\Tevent,\Theta)
    \nonumber\\
    &\quad=
    p(y_1)
    \prod_{j=2}^{L}
    p(y_j\mid y_{j-1},\theta_{z_{j-1}},\Delta_j).
    \label{eq:exact_grid_path_density}
\end{align}
Multiplying by \(p(x\mid y_{1:L},\tau^2)\) gives Eq.~\eqref{eq:y_conditional}.

\subsubsection{Regime Conditional Distribution}
\label{app:proof_regime_conditional}

The \((Z,Y_{\setminus\Tall})\) block factors as
\begin{align}
    &p(z,y_{\setminus\Tall}\mid y_{\Tall},\Tevent,\Theta,Q,\tau^2,x)
    \nonumber\\
    &\quad=p(z\mid y_{\Tall},\Tevent,\Theta,Q,\tau^2,x)
    \nonumber\\
    &\qquad{}\times p(y_{\setminus\Tall}\mid y_{\Tall},z,\Tevent,\Theta).
    \label{eq:exact_z_block}
\end{align}
Uniformization~\citep{rao2013} and identity transitions at observation-only times give
\begin{align}
    &p(\Tevent,z_{1:L}\mid Q)
    \nonumber\\
    &\quad=
    p(\Tevent\mid\Omega(Q))
    p(z_1)
    \prod_{j=2}^{L}
    A_j(z_{j-1},z_j).
    \label{eq:exact_uniformized_regime_density}
\end{align}
Here, \(p(\Tevent\mid\Omega(Q))\) is the Poisson-process density.
Bayes' rule with Eqs.~\eqref{eq:exact_uniformized_regime_density} and~\eqref{eq:exact_grid_path_density} expresses the first factor in Eq.~\eqref{eq:exact_z_block} as Eq.~\eqref{eq:exact_regime_conditional}.

\subsubsection{Parameter Updates}

The \((\Theta,Y_{\setminus\Tall})\) block factors as
\begin{align}
    &p(\Theta,y_{\setminus\Tall}\mid
    y_{\Tall},z,\Tevent,Q,\tau^2,x)
    \nonumber\\
    &\quad=
    p(\Theta\mid y_{\Tall},z,\Tevent)
    \nonumber\\
    &\qquad{}\times
    p(y_{\setminus\Tall}\mid y_{\Tall},z,\Tevent,\Theta).
    \label{eq:exact_theta_block}
\end{align}
Using Eq.~\eqref{eq:exact_grid_path_density} and the independent prior for \(\Theta\) gives Eq.~\eqref{eq:theta_conditional} for the first factor.
The \(\tau^2\) update in Eq.~\eqref{eq:tau_conditional} depends on \(Y\) only through \(Y_{\Tobs}\).

Marginalizing \(\Tcand\) recovers Eq.~\eqref{eq:exact_target}, under which \(Q\mid z\) follows Eq.~\eqref{eq:Q_conditional}.
After this collapsed update, new candidate times must be generated using the updated \(Q\) before any further grid-dependent update~\citep{vandyk2008}.

\subsubsection{Trimming and Invariance Across Iterations}

Within a sweep, each draw of \(Y_{\setminus\Tall}\) is unused by the subsequent marginal update; the \(\tau^2\) and \(Q\) updates also require no such values.

Every updated true-jump time lies in \(\Tevent\), implying that the updated \(\mathcal T_{\mathrm{ret}}\) is a subset of the current \(\Tall\).
Equation~\eqref{eq:exact_grid_augmentation} resamples the path outside this retained set before any values at new candidate times are used.
Thus, the final intervening-path draw of the previous sweep is also unused.

Trimming these draws preserves the retained-variable transition kernel~\citep{vandyk2008}.
The block updates and augmentation preserve the full-path posterior marginal; hence, at sweep boundaries, the invariant distribution is the marginal of Eq.~\eqref{eq:exact_target} on
\begin{align}
    (Z,Y_{\mathcal T_{\mathrm{ret}}},\Theta,\tau^2,Q),
    \qquad
    \mathcal T_{\mathrm{ret}}=\Tobs\cup\Ttrue.
\end{align}
In particular, its \((Z,\Theta,\tau^2,Q)\) marginal is the desired posterior marginal.

For fixed \(\Tevent\), the retained conditionals share the same marginalized posterior.
The \(Y_{\Tall}\) and \(\Theta\) draws may therefore be replaced by Markov transitions preserving their respective conditionals, including NUTS with fixed tuning~\citep{liu2001,hoffman2014}.

\section{Proof of Proposition~\modepropositionnumber}
\label{app:proof_prop2}

\begin{proof}
Fix a realized grid \(\mathcal G\) and its associated \(y_{1:L}\) and \(z_{1:L}\).
Because the exact and FNLE-based conditional densities differ only in their transition-density factors, they can be written with a common multiplicative convention as
\begin{align*}
    q_{\mathcal G}(\Theta)
    &=p(\Theta)
    \prod_{j=2}^{L}
    p(y_j\mid y_{j-1},h_{z_{j-1}}(\Theta),\Delta_j),\\
    \hat q_{\mathcal G}(\Theta)
    &=p(\Theta)
    \prod_{j=2}^{L}
    \phat(y_j\mid y_{j-1},h_{z_{j-1}}(\Theta),\Delta_j).
\end{align*}
Any omitted factors are independent of \(\Theta\) and therefore do not affect either conditional mode.
The common prior factor cancels from the difference between the normalized log densities, and we have
\begin{align*}
    e_{\mathcal G}(\Theta)
    &:=\hat\ell_{\mathcal G}(\Theta)
    -\ell_{\mathcal G}(\Theta)\\
    &=\frac{1}{T}
    \sum_{j=2}^{L}\eta_j(\Theta).
\end{align*}
For arbitrary functions \(f_2,\ldots,f_L\) on \(\mathcal K\),
\begin{align*}
    \operatorname{osc}_{\mathcal K}
    \left(\sum_{j=2}^{L}f_j\right)
    &\leq
    \sum_{j=2}^{L}
    \operatorname{osc}_{\mathcal K}(f_j).
\end{align*}
Applying this inequality to the local errors and using the first condition in Eq.~\eqref{eq:prop2_conditional_conditions} yields
\begin{align}
    \operatorname{osc}_{\mathcal K}(e_{\mathcal G})
    &\leq
    \frac{1}{T}
    \sum_{j=2}^{L}
    \operatorname{osc}_{\mathcal K}(\eta_j)
    \nonumber\\
    &\leq
    \frac{(L-1)\delta}{T}.
    \label{eq:prop2_accumulated_error}
\end{align}

The maximizing property of \(\hat\Theta_{\mathcal G}\) gives
\begin{align*}
    \hat\ell_{\mathcal G}
    (\hat\Theta_{\mathcal G})
    &\geq
    \hat\ell_{\mathcal G}
    (\Theta_{\mathcal G}^{\star}).
\end{align*}
Substituting \(\hat\ell_{\mathcal G}=\ell_{\mathcal G}+e_{\mathcal G}\) and rearranging gives
\begin{align*}
    \ell_{\mathcal G}(\Theta_{\mathcal G}^{\star})
    -\ell_{\mathcal G}(\hat\Theta_{\mathcal G})
    &\leq
    e_{\mathcal G}(\hat\Theta_{\mathcal G})
    -e_{\mathcal G}(\Theta_{\mathcal G}^{\star})\\
    &\leq
    \operatorname{osc}_{\mathcal K}(e_{\mathcal G}).
\end{align*}
Combining the second condition in Eq.~\eqref{eq:prop2_conditional_conditions} with Eq.~\eqref{eq:prop2_accumulated_error} therefore gives
\begin{align*}
    \frac{c}{2}D_{\mathcal G}^{2}
    &\leq
    \ell_{\mathcal G}(\Theta_{\mathcal G}^{\star})
    -\ell_{\mathcal G}(\hat\Theta_{\mathcal G})\\
    &\leq
    \frac{(L-1)\delta}{T}.
\end{align*}
Consequently, every realized grid satisfies the pointwise bound
\begin{align}
    D_{\mathcal G}
    &\leq
    \left\{
    \frac{2(L-1)\delta}{cT}
    \right\}^{1/2}.
    \label{eq:prop2_realized_grid_bound}
\end{align}

It remains to average Eq.~\eqref{eq:prop2_realized_grid_bound} over the joint uniformization construction.
Let \(M=|\Tevent|\) be the number of uniformization events.
Conditional on \(Q\), the event process is a Poisson process with rate \(\Omega(Q)\), and hence
\begin{align*}
    M\mid Q
    &\sim
    \Poisson\{\Omega(Q)T\}.
\end{align*}
The Poisson event times are almost surely distinct from one another and from the \(N\) fixed observation times, and \(L=N+M\) almost surely and
\begin{align}
    \E_{\mathrm{unif}}[L-1\mid Q]
    &=N-1+\Omega(Q)T.
    \label{eq:prop2_expected_grid_size}
\end{align}
The pointwise bound in Eq.~\eqref{eq:prop2_realized_grid_bound} is uniform in the associated fixed \(y\) and \(z\), and only the distribution of \(L\) is needed below.
Taking conditional expectations in Eq.~\eqref{eq:prop2_realized_grid_bound} and applying Jensen's inequality to the concave square-root function gives
\begin{align*}
    \E_{\mathrm{unif}}[D_{\mathcal G}\mid Q]
    &\leq
    \left(\frac{2\delta}{cT}\right)^{1/2}
    \E_{\mathrm{unif}}
    [(L-1)^{1/2}\mid Q]\\
    &\leq
    \left\{
    \frac{2\delta}{cT}
    \E_{\mathrm{unif}}[L-1\mid Q]
    \right\}^{1/2}\\
    &=
    \left[
    \frac{2\delta}{c}
    \left\{
    \frac{N-1}{T}+\Omega(Q)
    \right\}
    \right]^{1/2},
\end{align*}
where the final equality follows from Eq.~\eqref{eq:prop2_expected_grid_size}.
This is Eq.~\eqref{eq:prop2_expected_mode_stability} and completes the proof.
\end{proof}

\section{Algorithmic Details and Extensions}

\subsection{Sampling State Values at Candidate Times with SIR}
\label{app:sir_candidate_states}
Before applying the NUTS update to \(y\), we use sampling-importance-resampling (SIR) to assign latent values \(\{y_t\}_{t\in\Tcand}\).
Using local indices on the combined time grid containing the newly sampled \(\Tcand\), consider a consecutive block \(y_{a:b}\) of current state values at candidate times with fixed neighboring values \(y_{a-1}\) and \(y_{b+1}\) at adjacent points of \(\Tobs\cup\Ttrue\).
For particle \(p=1,
\ldots,P\), we simulate
\begin{align}
    y_j^{(p)}
    &\sim
    \phat(\cdot\mid y_{j-1}^{(p)},\theta_{z_{j-1}},\Delta_j),
    \qquad j=a,\ldots,b,
\end{align}
with \(y_{a-1}^{(p)}=y_{a-1}\).
Because the proposal is the forward FNLE transition-density approximation from the left endpoint, the SIR weight is
\begin{align}
    \widetilde{w}_p
    &= \phat(
    y_{b+1}\mid y_b^{(p)},
    \theta_{z_b},\Delta_{b+1}
    ),\\
    w_p
    &= \frac{\widetilde{w}_p}{\sum_{q=1}^P\widetilde{w}_q}.
\end{align}
We then draw \(A\sim\Cat(w_1,\ldots,w_P)\) and set \(y_{a:b}=y_{a:b}^{(A)}\).
This construction is applied to every interval between consecutive points of \(\Tobs\cup\Ttrue\) that contains at least one current candidate time.

\subsection{Conjugate Parameter Updates}
\label{app:conjugate_updates}

Let \(\tau_d^2\) denote the \(d\)-th element of \(\tau^2\).
Assume independent priors \(\tau_d^2\sim\InvGamma(\alpha_d^{\tau},\beta_d^{\tau})\) and \(q_{kh}\sim\GammaDist(\alpha^Q,\beta^Q)\) for \(k\neq h\), with the Gamma distribution parameterized by shape and rate.
Here, \(\InvGamma(a,b)\) denotes the inverse-Gamma distribution with shape \(a\) and scale \(b\), whose density is proportional to \(v^{-a-1}\exp(-b/v)\) for \(v>0\).
Let
\begin{align}
    R_d
    =\sum_{i=1}^{N}
    \left(x_{i,d}-y_{t_i^{\mathrm{obs}},d}\right)^2.
\end{align}
The conjugate observation-noise update used in Algorithm~\ref{alg:ssde_sampler} is
\begin{align}
    \tau_d^2\mid x,y
    \sim
    \InvGamma\!\left(
    \alpha_d^{\tau}+\frac{N}{2},
    \beta_d^{\tau}+\frac{R_d}{2}
    \right).
    \label{eq:tau_conditional}
\end{align}
Let \(n_{kh}\) be the number of true jumps from \(k\) to \(h\), and let
\begin{align}
    m_k
    =\int_0^{T}\Ind\{Z_t=k\}\dd t
\end{align}
be the total holding time in regime \(k\).
After marginalizing the virtual jumps, the conjugate update is
\begin{align}
    q_{kh}\mid z
    \sim
    \GammaDist(\alpha^Q+n_{kh},\beta^Q+m_k),
    \qquad k\neq h.
    \label{eq:Q_conditional}
\end{align}
The diagonal entries are set to
\begin{align}
    q_{kk}=-\sum_{h\neq k}q_{kh}.
    \label{eq:Q_diagonal}
\end{align}

\subsection{Extension to Multiple Time Series}
\label{app:multiple_series}

Suppose that \(S\) time series are conditionally independent given the shared parameters \(\Theta,\tau^2,Q\).
Series \(s\) has its own observations, latent state process, regime path, and combined time grid, denoted by \(x^{(s)},Y^{(s)},Z^{(s)}\), and \(\Tall^{(s)}\), respectively.
Let \(N_s\) be its number of observations, \(L_s\) its number of combined grid points, and \([0,T_s]\) its observation interval.
We use \(y_j^{(s)},z_j^{(s)}\), and \(\Delta_j^{(s)}\) for its latent states and regimes at the grid points, and its interval lengths, respectively.

The initialization and the candidate-grid, latent-state, and regime updates in Algorithm~\ref{alg:ssde_sampler} are performed separately for each series.
The shared parameters are then updated using the contributions from all series, with each prior included only once.
The FNLE-based conditional density corresponding to Eq.~\eqref{eq:theta_conditional} becomes
\begin{align}
    \hat p\bigl(\Theta\mid
    \{y^{(s)},z^{(s)},\Tall^{(s)}\}_{s=1}^{S}\bigr)
    &\propto p(\Theta)\prod_{s=1}^{S}\prod_{j=2}^{L_s}
    \phat\bigl(y_j^{(s)}\mid y_{j-1}^{(s)},
    h_{z_{j-1}^{(s)}}(\Theta),\Delta_j^{(s)}\bigr).
\end{align}

For the observation-noise update in Eq.~\eqref{eq:tau_conditional}, replace \(N\) by \(\sum_{s=1}^{S}N_s\) and \(R_d\) by \(\sum_{s=1}^{S}R_d^{(s)}\), where
\begin{align}
    R_d^{(s)}
    =\sum_{i=1}^{N_s}
    \left(x_{i,d}^{(s)}
    -y_{t_i^{\mathrm{obs},(s)},d}^{(s)}\right)^2
\end{align}
is the residual sum of squares for component \(d\) in series \(s\).
For Eq.~\eqref{eq:Q_conditional}, replace \(n_{kh}\) and \(m_k\) by \(\sum_{s=1}^{S}n_{kh}^{(s)}\) and \(\sum_{s=1}^{S}m_k^{(s)}\), where \(n_{kh}^{(s)}\) counts the true jumps from \(k\) to \(h\) in series \(s\), and
\begin{align}
    m_k^{(s)}=\int_0^{T_s}\Ind\{Z_t^{(s)}=k\}\,\dd t
\end{align}
is its holding time in regime \(k\).

\subsection{Computational Complexity}
\label{app:computational_complexity}

For one sweep with \(L\) combined grid points, \(K\) regimes, and \(P\) SIR particles, constructing the emission factors requires \(K(L-1)\) FNLE density evaluations, followed by \(O(LK^2)\) arithmetic operations for FFBS~\citep{rao2013}.
SIR requires \(O(P|\Tcand|)\) conditional FNLE samples and density evaluations across all candidate blocks.
The NUTS cost depends on its trajectory-dependent number of gradient evaluations, each involving \(O(L)\) local FNLE terms, together with observation and prior terms~\citep{hoffman2014}.
These costs exclude offline FNLE training.

\section{Experimental Settings}
\label{app:experimental_settings}

This appendix records the settings for the synthetic-data experiments (OU comparison, LV, gene-expression CLE, and susceptible--infected--recovered) and the real-data LV analysis.

\subsection{Settings for the Main Experiments}

The synthetic-data LV, gene-expression CLE, and susceptible--infected--recovered experiments, along with the real-data LV analysis, use two regimes and observe every state component with Gaussian errors as in Eq.~\eqref{eq:obs_model}.

\subsubsection{FNLE Training}
\label{app:fnle_training}

The time-conditioned FNLE training dataset is
\begin{align}
    \mathcal{D}_{\mathrm{train}}
    =
    \{(y_r^{\mathrm{prev}},y_r^{\mathrm{next}},
    \theta_r,\Delta_r)\}_{r=1}^{R_{\mathrm{train}}},
\end{align}
where \(R_{\mathrm{train}}\) is the number of simulated transitions and \(r\) indexes training examples rather than positions in a trajectory.
The density estimator is trained by maximizing the conditional log-likelihood
\begin{align}
    \max_{\phi}\;
    \sum_{r=1}^{R_{\mathrm{train}}}
    \log\phat\left(
    y_r^{\mathrm{next}}\mid
    y_r^{\mathrm{prev}},\theta_r,\Delta_r
    \right).
    \label{eq:time_conditioned_fnle_objective}
\end{align}
Omitting elapsed-time conditioning gives the standard FNLE training objective.

A separate time-conditioned FNLE model was trained for each experiment.
We used preliminary MCMC draws to check whether the FNLE parameter training ranges in Sections~\ref{app:synthetic_inference_settings} and~\ref{app:real_data_settings} covered the parameter regions relevant to inference.
Where needed, we expanded the ranges, retrained FNLE, and reran MCMC.
These adjustments refine the numerical likelihood approximation without changing the SSDE model or prior.

A natural cubic spline through that experiment's observations was evaluated on an Euler--Maruyama grid~\citep{kloeden2013} with step size \(h=0.01\) to obtain a reference path.
Negative spline values were truncated at zero.
For every training example, each component of \(\log\theta\) was sampled uniformly between the logarithms of the corresponding physical-scale bounds reported below, and \(\theta\) was then obtained by exponentiation.
For each training example, one location was selected from the reference path.
Isotropic Gaussian noise with standard deviation 0.32 was added to this state.
If the perturbed state left \([0,10^4]^D\), the noise was redrawn until the state lay inside this domain.

The maximum elapsed time required for density evaluation is bounded by the largest observation interval, which determines the required training range.
The number of numerical transitions, \(n\), was sampled uniformly from \(\{1,\ldots,n_{\max}\}\), and the elapsed time was \(\Delta=nh\).
The simulator was advanced by \(n\) Euler--Maruyama steps, and FNLE was trained on the scaled increment
\begin{align}
    R
    =\frac{Y_{t+\Delta}-Y_t}{\sqrt{\Delta}}
    \label{eq:scaled_increment}
\end{align}
conditional on \((\theta,Y_t,\Delta)\).
When the learned density of \(R\) is evaluated as a transition density for \(Y_{t+\Delta}\), the change of variables gives
\begin{align}
    &\phat(y^{\mathrm{next}}\mid
    y^{\mathrm{prev}},\theta,\Delta)
    \nonumber\\
    &\quad=
    \Delta^{-D/2}
    \hat p_{\phi,R}\!\left(
    \frac{y^{\mathrm{next}}-y^{\mathrm{prev}}}{\sqrt{\Delta}}
    \middle|
    y^{\mathrm{prev}},\theta,\Delta
    \right).
    \label{eq:scaled_increment_change_of_variables}
\end{align}
Thus, Eq.~\eqref{eq:scaled_increment} is a target reparameterization and does not change the underlying transition model.
For the OU experiment, the two FNLE training sets shared identical conditioning inputs \((y^{\mathrm{prev}},\theta,\Delta)\) and differed only in how \(y^{\mathrm{next}}\) was generated.
For approximate MCMC with FNLE (exact-transition training), it was sampled from the exact transition density in Appendix~\ref{app:synthetic_ou}; for approximate MCMC with FNLE (Euler--Maruyama training), it was generated by Euler--Maruyama.

Every FNLE model was a neural spline flow~\citep{durkan2019} with 50 hidden features, five transforms, and ten spline bins.
Each synthetic-data training set contained 100,000 independently sampled parameter settings, and the real-data training set contained 800,000, with one simulated transition per setting.
Training used a batch size of 256, a learning rate of \(5\times10^{-4}\), at most 100 epochs, and early stopping after 20 epochs without validation-loss improvement.
The density estimator and batch size were selected through preliminary experiments, and maximum epoch counts were set according to the available computational budget.
All other FNLE training hyperparameters used the \texttt{sbi} defaults, including reserving 10\% of the simulated transitions for validation.
The experiment-specific transition-length settings and final epoch counts are given in Table~\ref{tab:nle_transition_lengths}.

\begin{table*}[t]
    \caption{Experiment-Specific FNLE Training Settings.}
    \label{tab:nle_transition_lengths}
    \begin{center}
    \small
    \begin{tabular}{lcc}
        \toprule
        Experiment & \(n_{\max}\) & \(n_{\max}h\) \\
        \midrule
        Synthetic-data OU (approx. MCMC with Euler--Maruyama FNLE) & 100 & 1.00 \\
        Synthetic-data LV & 100 & 1.00 \\
        Synthetic-data gene-expression CLE & 25 & 0.25 \\
        Synthetic-data susceptible--infected--recovered & 20 & 0.20 \\
        Real-data LV & 98 & 0.98 \\
        \bottomrule
    \end{tabular}
    \end{center}
\end{table*}

\subsubsection{Prior Settings}
\label{app:prior_settings}

SDE parameter priors and parameter-sharing structures are given in Sections~\ref{app:synthetic_inference_settings} and~\ref{app:real_data_settings}.
Each prior is evaluated once per unique parameter.
Here, \(\operatorname{LogNormal}(m,s)\) denotes a distribution whose natural logarithm is Gaussian with mean \(m\) and standard deviation \(s\).
The initial regime probabilities were \((1/2,1/2)\), and the initial state had an improper uniform prior, \(p(y_1)\propto 1\) on \(\mathbb R^D\).

The observation-variance priors are \(\tau_d^2\sim\InvGamma(\alpha_d^{\tau},\beta_d^{\tau})\), with the experiment-specific settings listed in Table~\ref{tab:noise_priors}.

\begin{table*}[t]
    \caption{Priors for the Observation Variances}
    \label{tab:noise_priors}
    \begin{center}
    \small
    \begin{tabular}{lll}
        \toprule
        Experiment & Component & Prior \\
        \midrule
        Synthetic-data OU & scalar state & \(\tau^2\sim\InvGamma(10^{-3},10^{-3})\) \\
        Synthetic-data LV  & predator & \(\tau_1^2\sim\InvGamma(10^{-3},10^{-3})\) \\
        Synthetic-data LV  & prey & \(\tau_2^2\sim\InvGamma(10^{-3},10^{-3})\) \\
        Synthetic-data gene-expression CLE  & mRNA & \(\tau_1^2\sim\InvGamma(10^{-3},10^{-3})\) \\
        Synthetic-data gene-expression CLE  & protein fluorescence & \(\tau_2^2\sim\InvGamma(10^{-3},10^{-3})\) \\
        Synthetic-data susceptible--infected--recovered  & susceptible & \(\tau_1^2\sim\InvGamma(10^{-3},10^{-3})\) \\
        Synthetic-data susceptible--infected--recovered  & recovered & \(\tau_2^2\sim\InvGamma(10^{-3},10^{-3})\) \\
        Real-data LV & \textit{Didinium} & \(\tau_1^2\sim\InvGamma(10^{-3},10^{-3})\) \\
        Real-data LV & \textit{Paramecium} & \(\tau_2^2\sim\InvGamma(10^{-3},10^{-3})\) \\
        \bottomrule
    \end{tabular}
    \end{center}
\end{table*}

In all experiments, the off-diagonal CTMC rates independently followed \(\GammaDist(2,20)\) in the shape--rate parameterization.

\subsubsection{MCMC Settings and Initialization}
\label{app:mcmc_settings}

The deterministic uniformization rate used \(c_\Omega=3\) and a numerical safeguard \(\varepsilon=10^{-6}\) in
\begin{align}
    \Omega(Q)
    =\max\left\{
    c_\Omega\max_k(-q_{kk}),
    \max_k(-q_{kk})+\varepsilon,
    \varepsilon
    \right\}.
\end{align}
After computing \(B=I_K+Q/\Omega(Q)\), negative entries caused by numerical error were clipped to zero and each row was renormalized.
Newly inserted candidate-time states were sampled by sampling-importance-resampling with 100 particles.
The continuous path and nonredundant SDE parameters were updated conditionally using Pyro NUTS.
Before the production runs, we used Pyro's step-size adaptation in preliminary runs and fixed the resulting step sizes. Table~\ref{tab:nuts_step_sizes} lists the step sizes for the state-path and parameter updates by experiment and sampler.
The production chains used one NUTS draw per sweep, with no warm-up or further step-size adaptation.
In all experiments, the maximum tree depth was 5 for the state-path update and 3 for the parameter update; these values were selected in preliminary experiments based on the observed divergence rates and sample displacements.
The regime path was updated by log-space forward filtering and backward sampling.

\begin{table*}[t]
    \caption{Fixed NUTS Step Sizes by Experiment and Sampler}
    \label{tab:nuts_step_sizes}
    \begin{center}
    \small
    \begin{tabular*}{\textwidth}{@{\extracolsep{\fill}}p{0.68\textwidth}cc@{}}
        \toprule
        Experiment & State path \(y_{1:L}\) & Parameters \(\Theta\) \\
        \midrule
        Synthetic-data OU (exact MCMC) & \(2^{-8}\) & \(2^{-5}\) \\
        Synthetic-data OU (approx. MCMC with exact FNLE) & \(2^{-8}\) & \(2^{-5}\) \\
        Synthetic-data OU (approx. MCMC with Euler--Maruyama FNLE) & \(2^{-8}\) & \(2^{-5}\) \\
        Synthetic-data LV & \(2^{-9}\) & \(2^{-8}\) \\
        Synthetic-data gene-expression CLE & \(2^{-9}\) & \(2^{-8}\) \\
        Synthetic-data susceptible--infected--recovered & \(2^{-9}\) & \(2^{-8}\) \\
        Real-data LV & \(2^{-9}\) & \(2^{-5}\) \\
        \bottomrule
    \end{tabular*}
    \end{center}
\end{table*}

At initialization, SDE parameters were sampled from their experiment-specific priors conditional on lying inside the corresponding FNLE ranges.
Shared components were sampled once and copied across regimes.
Every off-diagonal CTMC rate and observation variance was sampled from its prior.
The initial regime path was simulated from an auxiliary symmetric CTMC with \(q_{12}=q_{21}=0.2\) for the synthetic-data OU and LV experiments, 0.4 for the synthetic-data gene-expression CLE experiment, and 0.5 for the synthetic-data susceptible--infected--recovered and real-data LV experiments.
This generator was used only to initialize the regime path, independently of the initial \(Q\) drawn from its prior.
The initial continuous path was obtained by linear interpolation of the observations.
FNLE training and the synthetic-data chains used random seed zero; the four real-data chains used seeds zero through three.
Chain lengths are summarized in Table~\ref{tab:chain_lengths}.
For trajectory summaries, each sampled path was linearly interpolated between its grid points onto a common plotting grid before computing pointwise posterior credible intervals.

\begin{table*}[t]
    \caption{MCMC Chain Counts and Lengths Used for Posterior Summaries.
    Sweep, burn-in, and retained-draw counts are per chain.}
    \label{tab:chain_lengths}
    \begin{center}
    \small
    \begin{tabular}{lrrrr}
        \toprule
        Experiment & Chains & Sweeps & Burn-in & Retained \\
        \midrule
        Synthetic-data OU (exact MCMC) & 4 & 10,000 & 5,000 & 5,000 \\
        Synthetic-data OU (approx. MCMC with exact FNLE) & 4 & 10,000 & 5,000 & 5,000 \\
        Synthetic-data OU (approx. MCMC with Euler--Maruyama FNLE) & 4 & 10,000 & 5,000 & 5,000 \\
        Synthetic-data LV & 4 & 10,000 & 5,000 & 5,000 \\
        Synthetic-data gene-expression CLE & 4 & 10,000 & 5,000 & 5,000 \\
        Synthetic-data susceptible--infected--recovered & 4 & 10,000 & 5,000 & 5,000 \\
        Real-data LV & 4 & 10,000 & 5,000 & 5,000 \\
        \bottomrule
    \end{tabular}
    \end{center}
\end{table*}

\subsection{Synthetic-Data Experiments}

\subsubsection{Ornstein--Uhlenbeck Experiment}
\label{app:synthetic_ou}

The OU experiment used Eq.~\eqref{eq:synthetic_ou_model}.
On an interval of length \(\Delta\) with regime \(k\) fixed, define
\(m_k(x,u)=\mu_k+e^{-\kappa u}(x-\mu_k)\) and
\(v(u)=\sigma^2(1-e^{-2\kappa u})/(2\kappa)\).
The exact transition and one-time bridge distributions are
\begin{align}
    p(y\mid x,\theta_k,\Delta)
    &=\Normal\bigl(y;m_k(x,\Delta),v(\Delta)\bigr),
    \\
    Y_u\mid(Y_0=x,Y_\Delta=y,\theta_k)
    &\sim\Normal\left(
        m_k(x,u)
        +e^{-\kappa(\Delta-u)}\frac{v(u)}{v(\Delta)}
        \left[y-m_k(x,\Delta)\right],
        \frac{v(u)v(\Delta-u)}{v(\Delta)}
    \right),\quad 0<u<\Delta.
\end{align}
For multiple interior times, the joint bridge distribution follows from Eq.~\eqref{eq:exact_candidate_bridge} and is multivariate Gaussian.
Exact MCMC uses the transition density directly, while approximate MCMC with FNLE (exact-transition training) uses it only to generate FNLE training transitions and uses the learned density at inference.
The exact bridge distribution is used only by exact MCMC; both approximate samplers use the FNLE-based SIR bridge update.
Conditional on this path, the state was simulated by Euler--Maruyama with step size \(h=0.01\) and independent Brownian increments.

The initial state was \(Y_{0}=0.0\).
The observation standard deviations were \(\tau_1=\tau_2=0.01\).
The observation period was \(T=100\), with 101 observations at spacing 1.0.

\subsubsection{Switching Lotka--Volterra Experiment}
\label{app:synthetic_lv}

The LV experiment used Eq.~\eqref{eq:synthetic_lv_model}.
The regime path was simulated from a two-state CTMC with initial probabilities \((1/2,1/2)\), exact holding times, and exact jump destinations.
Conditional on this path, the state was simulated by Euler--Maruyama with step size \(h=0.01\) and independent Brownian increments.
After each step, the state was restricted to \([0,10^4]^2\) componentwise.

The initial state was \((Y_{0,1},Y_{0,2})=(1.0,0.5)\), and the data-generating CTMC rates were \(q_{12}=q_{21}=0.03\).
The observation standard deviations were \(\tau_1=\tau_2=0.01\).
The observation period was \(T=100\), with 101 observations at spacing 1.0.

\subsubsection{Switching Gene-Expression CLE Experiment}
\label{app:synthetic_cle}

The gene-expression CLE experiment used Eq.~\eqref{eq:synthetic_cle_model}.
The regime path was simulated from a two-state CTMC with initial probabilities \((1/2,1/2)\), exact holding times, and exact jump destinations.
Conditional on this path, the state was simulated by Euler--Maruyama with step size \(h=0.01\) and independent Brownian increments.
After each step, the state was restricted to \([0,10^4]^2\) componentwise.
To see the change of scale, let \(P_t\) be the continuous approximation to the protein copy number, with translation rate \(\gamma/c^2\) per mRNA molecule and degradation rate \(\delta\).
Its chemical Langevin equation is
\begin{align}
    \dd P_t
    &=\left(\frac{\gamma}{c^2}Y_{t,1}-\delta P_t\right)\dd t
    \nonumber\\
    &\quad+\sqrt{\frac{\gamma}{c^2}Y_{t,1}+\delta P_t}\,\dd W_{t,2}.
\end{align}
The transformation \(Y_{t,2}=c^2P_t\), with \(c>0\), gives Eq.~\eqref{eq:synthetic_cle_model}.
Thus, \(\gamma\) is the fluorescence-scaled production coefficient, and \(c^2\) converts protein copy number to fluorescence intensity.

The initial state was \((Y_{0,1},Y_{0,2})=(100,0.5)\), and the data-generating CTMC rates were \(q_{12}=q_{21}=0.05\).
The observation standard deviations were \(\tau_1=1.0\) and \(\tau_2=0.01\).
The observation period was \(T=50\), with 201 observations at spacing 0.25.

\subsubsection{Switching Susceptible--Infected--Recovered Epidemic Experiment}
\label{app:synthetic_sir}

The susceptible--infected--recovered experiment used Eq.~\eqref{eq:synthetic_sir_model}, with fixed total population \(N_{\mathrm{pop}}=1000\).
The regime path had a single prescribed switch at the midpoint of the observation period.
Conditional on this path, the state was simulated by Euler--Maruyama with step size \(h=0.01\) and independent Brownian increments.
After each step, the state was restricted to \([0,10^4]^2\) componentwise and projected onto \(\{(s,r):s\geq0,\ r\geq0,\ s+r\leq N_{\mathrm{pop}}\}\).

The initial state was \((Y_{0,1},Y_{0,2})=(990,0)\), implying \(I_0=10\).
The regime path was fixed to regime 1 on \([0,20)\) and regime 2 on \([20,40]\).
The observation standard deviations were \(\tau_1=\tau_2=0.5\), and negative noisy observations were truncated at zero.
The observation period was \(T=40\), with 201 observations at spacing 0.20.
Both state components were supplied to inference and plotted.
The infected population was obtained as \(I_t=N_{\mathrm{pop}}-Y_{t,1}-Y_{t,2}\).

\subsubsection{Data-Generating Parameters and Inference Settings}
\label{app:synthetic_inference_settings}

Table~\ref{tab:synthetic_truth} lists the data-generating SDE parameters, with regime-specific values ordered by regime; Table~\ref{tab:synthetic_inference_settings} gives the FNLE ranges, parameter-sharing structures, and priors.

\begin{table*}[t]
    \caption{Data-Generating SDE Parameters for the Synthetic Experiments. Regime-specific values are listed for regimes 1 and 2.}
    \label{tab:synthetic_truth}
    \begin{center}
    \small
    \begin{tabular*}{\textwidth}{@{\extracolsep{\fill}}llccc@{}}
        \toprule
        Model & Parameter & Regime 1 & Regime 2 & Structure \\
        \midrule
        \multirow{3}{*}{OU} & \(\kappa\) & 0.15 & 0.15 & shared \\
        & \(\mu\) & -1.2 & 1.2 & regime-specific \\
        & \(\sigma\) & 0.05 & 0.05 & shared \\
        \midrule
        \multirow{6}{*}{LV} & \(\alpha\) & 1.0 & 0.5 & regime-specific \\
        & \(\beta\) & 1.0 & 0.5 & regime-specific \\
        & \(\gamma\) & 1.0 & 0.2 & regime-specific \\
        & \(\delta\) & 1.0 & 0.2 & regime-specific \\
        & \(\sigma_1\) & 0.05 & 0.05 & shared \\
        & \(\sigma_2\) & 0.05 & 0.05 & shared \\
        \midrule
        \multirow{5}{*}{Gene-expression CLE} & \(\rho\) & 100 & 240 & regime-specific \\
        & \(\beta\) & 1.0 & 1.0 & shared \\
        & \(\gamma\) & 0.004 & 0.004 & shared \\
        & \(\delta\) & 0.8 & 0.8 & shared \\
        & \(c\) & 0.03 & 0.03 & shared \\
        \midrule
        \multirow{2}{*}{Susceptible--infected--recovered} & \(\beta\) & 0.40 & 0.80 & regime-specific \\
        & \(\gamma\) & 0.10 & 0.20 & regime-specific \\
        \bottomrule
    \end{tabular*}
    \end{center}
\end{table*}

\begin{table*}[t]
    \caption{SDE Parameter Inference Settings for the Synthetic Experiments.
    Parameter ranges and priors are on the physical scale.}
    \label{tab:synthetic_inference_settings}
    \begin{center}
    \small
    \begin{tabular}{lllll}
        \toprule
        Model & Parameter & Structure & FNLE range & Prior \\
        \midrule
        \multirow{3}{*}{OU} & \(\kappa\) & shared & \([0.05,0.2]\) & \(\operatorname{LogNormal}(0,1)\) \\
        & \(\mu\) & regime-specific & \([-2.0,2.0]\) & \(\operatorname{Normal}(0,1)\) \\
        & \(\sigma\) & shared & \([0.01,0.1]\) & \(\operatorname{LogNormal}(0,1)\) \\
        \midrule
        \multirow{6}{*}{LV} & \(\alpha\) & regime-specific & \([0.01,2]\) & \(\operatorname{LogNormal}(0,1)\) \\
        & \(\beta\) & regime-specific & \([0.1,3]\) & \(\operatorname{LogNormal}(0,1)\) \\
        & \(\gamma\) & regime-specific & \([0.01,2]\) & \(\operatorname{LogNormal}(0,1)\) \\
        & \(\delta\) & regime-specific & \([0.1,3]\) & \(\operatorname{LogNormal}(0,1)\) \\
        & \(\sigma_1\) & shared & \([0.01,0.30]\) & \(\operatorname{LogNormal}(-1,1)\) \\
        & \(\sigma_2\) & shared & \([0.01,0.20]\) & \(\operatorname{LogNormal}(-1,1)\) \\
        \midrule
        \multirow{5}{*}{Gene-expression CLE} & \(\rho\) & regime-specific & \([50,400]\) & \(\operatorname{LogNormal}(4,1)\) \\
        & \(\beta\) & shared & \([0.5,1.5]\) & \(\operatorname{LogNormal}(0,1)\) \\
        & \(\gamma\) & shared & \([0.001,0.010]\) & \(\operatorname{LogNormal}(0,1)\) \\
        & \(\delta\) & shared & \([0.5,1.5]\) & \(\operatorname{LogNormal}(0,1)\) \\
        & \(c\) & shared & \([0.01,0.10]\) & \(\operatorname{LogNormal}(0,1)\) \\
        \midrule
        \multirow{2}{*}{Susceptible--infected--recovered} & \(\beta\) & regime-specific & \([0.05,2]\) & \(\operatorname{LogNormal}(0,1)\) \\
        & \(\gamma\) & regime-specific & \([0.02,2]\) & \(\operatorname{LogNormal}(0,1)\) \\
        \bottomrule
    \end{tabular}
    \end{center}
\end{table*}

For all synthetic-data experiments, regime labels were aligned to the truth by choosing the permutation that minimized the root-mean-square discrepancy between posterior mean log parameters and the data-generating log parameters.

\subsection{Real-Data Experiment}
\label{app:real_data_settings}

The real-data analysis used 71 paired observations of \textit{Didinium} and \textit{Paramecium} abundance from \citet{veilleux1979}.
The observations were sorted by time, the first time was shifted to zero, and both abundance variables were divided by 100 before fitting.
The observation times were irregular and spanned \([0,35.08]\) days.
The state was ordered as \(Y_t=(Y_{t,1},Y_{t,2})^\top\), where \(Y_{t,1}\) and \(Y_{t,2}\) denote \textit{Didinium} and \textit{Paramecium}, respectively.

The analysis used the LV model in Eq.~\eqref{eq:synthetic_lv_model}.
Table~\ref{tab:real_lv_inference} gives the parameter-sharing structure, FNLE ranges, and priors.

\begin{table*}[t]
    \caption{SDE Parameter Settings for the Real-Data LV Analysis.
    Parameter ranges and priors are on the physical scale.}
    \label{tab:real_lv_inference}
    \begin{center}
    \small
    \begin{tabular}{llll}
        \toprule
        Parameter & Structure & FNLE range & Prior \\
        \midrule
        \(\alpha\) & regime-specific & \([0.1,4]\) & \(\operatorname{LogNormal}(0,1)\) \\
        \(\beta\) & regime-specific & \([0.01,6]\) & \(\operatorname{LogNormal}(0,1)\) \\
        \(\gamma\) & regime-specific & \([0.1,8]\) & \(\operatorname{LogNormal}(0,1)\) \\
        \(\delta\) & regime-specific & \([0.01,5]\) & \(\operatorname{LogNormal}(0,1)\) \\
        \(\sigma_1\) & shared & \([0.01,0.8]\) & \(\operatorname{LogNormal}(-1,1)\) \\
        \(\sigma_2\) & shared & \([0.01,2]\) & \(\operatorname{LogNormal}(-1,1)\) \\
        \bottomrule
    \end{tabular}
    \end{center}
\end{table*}

For Figure~\ref{fig:real_data_inference_results}, samples from chains 0--3 were pooled after aligning the regime labels.

\subsubsection{Chain-Specific Posterior Distributions and Traces}
\label{app:real_data_diagnostics}

Figure~\ref{fig:real_data_multichain_trace} compares the parameter distributions and traces from all four real-data chains.

\begin{figure*}[p]
    \centering
    \includegraphics[width=\textwidth,height=0.84\textheight,keepaspectratio]{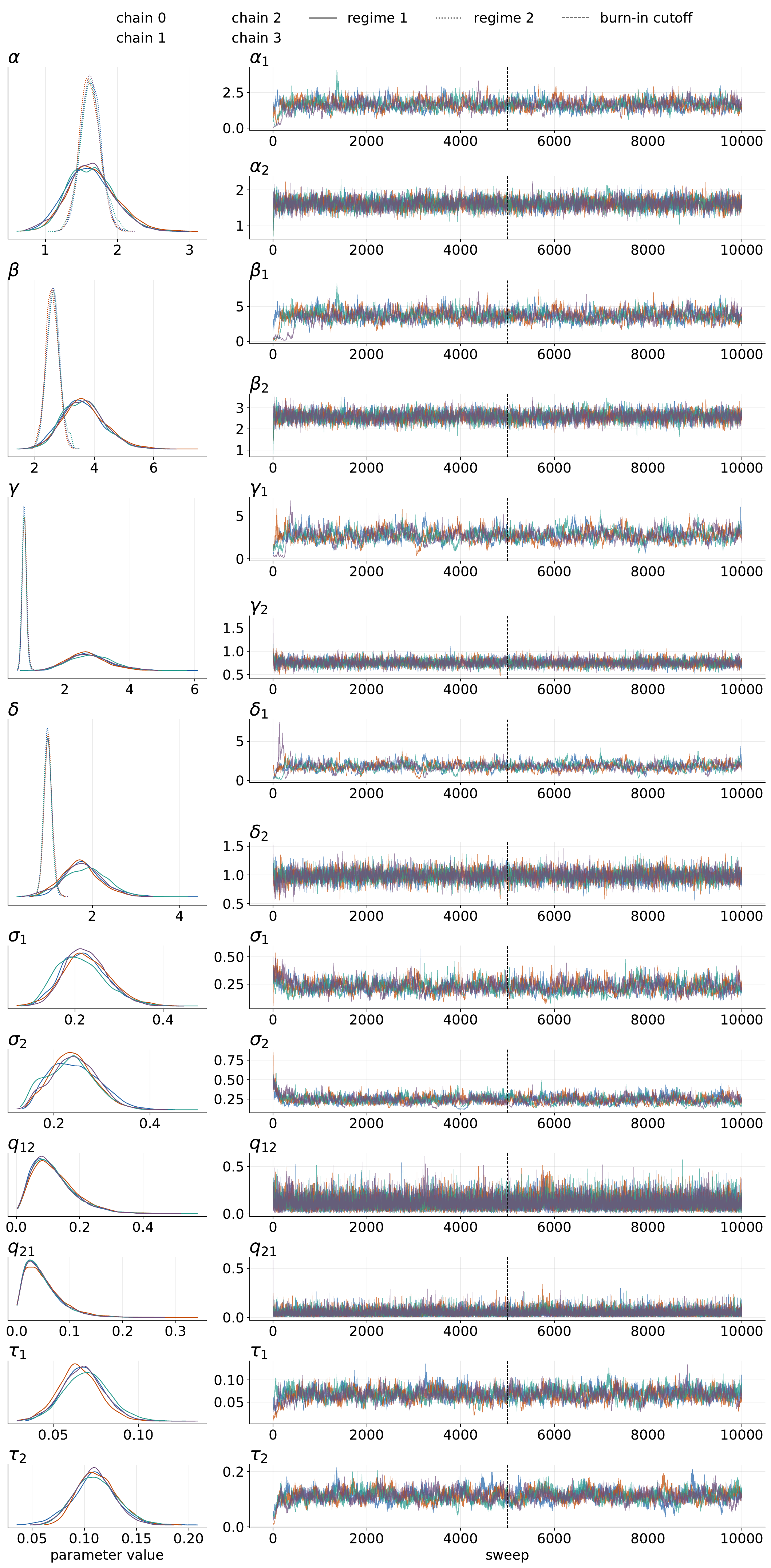}
    \caption{Multi-chain diagnostics for the real-data Lotka--Volterra analysis.
    Left panels show chain-specific marginal posterior density estimates, and right panels show traces over all 10,000 sweeps.
    Colors distinguish the four chains; solid and dotted lines distinguish regimes 1 and 2 for regime-specific parameters.
    Vertical dashed lines mark the burn-in cutoff at 5,000 sweeps.
    The observation-noise panels show standard deviations \(\tau_1,\tau_2\), not variances.}
    \label{fig:real_data_multichain_trace}
\end{figure*}

\clearpage
\subsection{MCMC Diagnostics}
\label{app:mcmc_diagnostics}

\begin{table*}[htbp]
    \caption{Per-parameter MCMC diagnostics for the three OU samplers on synthetic data.}
    \label{tab:ou_diagnostics}
    \begin{center}
    \small
    \begin{tabular}{llccc}
        \toprule
        OU sampler & Parameter & \(\hat R\) & Bulk ESS & Tail ESS \\
        \midrule
        \multirow{7}{*}{Exact MCMC}
        & \(\kappa\) & 1.015 & 198.4 & 1,004.2 \\
        & \(\mu_1\) & 1.004 & 1,187.8 & 6,199.0 \\
        & \(\mu_2\) & 1.005 & 1,478.3 & 5,050.6 \\
        & \(\sigma\) & 1.004 & 1,519.8 & 2,234.7 \\
        & \(q_{12}\) & 1.000 & 16,311.0 & 19,035.3 \\
        & \(q_{21}\) & 1.000 & 16,972.1 & 18,373.4 \\
        & \(\tau\) & 1.007 & 787.0 & 1,386.9 \\
        \midrule
        \multirow{7}{*}{\shortstack{Approx. MCMC with FNLE\\(exact-transition training)}}
        & \(\kappa\) & 1.011 & 323.9 & 2,294.4 \\
        & \(\mu_1\) & 1.002 & 5,909.5 & 13,911.3 \\
        & \(\mu_2\) & 1.006 & 792.1 & 3,312.8 \\
        & \(\sigma\) & 1.004 & 778.2 & 1,338.5 \\
        & \(q_{12}\) & 1.001 & 11,410.3 & 16,523.3 \\
        & \(q_{21}\) & 1.001 & 12,109.2 & 19,161.3 \\
        & \(\tau\) & 1.017 & 405.4 & 1,032.6 \\
        \midrule
        \multirow{7}{*}{\shortstack{Approx. MCMC with FNLE\\(Euler--Maruyama training)}}
        & \(\kappa\) & 1.021 & 229.2 & 1,903.9 \\
        & \(\mu_1\) & 1.006 & 1,497.8 & 4,256.4 \\
        & \(\mu_2\) & 1.004 & 2,159.5 & 7,742.9 \\
        & \(\sigma\) & 1.004 & 953.9 & 1,984.0 \\
        & \(q_{12}\) & 1.000 & 14,125.3 & 16,715.3 \\
        & \(q_{21}\) & 1.000 & 12,237.1 & 19,034.8 \\
        & \(\tau\) & 1.008 & 487.4 & 838.5 \\
        \bottomrule
    \end{tabular}
    \end{center}
\end{table*}

\begin{table*}[htbp]
    \caption{Per-parameter MCMC diagnostics for the synthetic Lotka--Volterra experiment.}
    \label{tab:lv_diagnostics}
    \begin{center}
    \small
    \begin{tabular}{lccc}
        \toprule
        Parameter & \(\hat R\) & Bulk ESS & Tail ESS \\
        \midrule
        \(\alpha_1\) & 1.006 & 421.2 & 1,090.5 \\
        \(\alpha_2\) & 1.007 & 406.2 & 777.6 \\
        \(\beta_1\) & 1.006 & 395.9 & 973.0 \\
        \(\beta_2\) & 1.006 & 385.9 & 771.1 \\
        \(\gamma_1\) & 1.028 & 217.1 & 893.4 \\
        \(\gamma_2\) & 1.038 & 143.4 & 301.5 \\
        \(\delta_1\) & 1.037 & 128.9 & 726.6 \\
        \(\delta_2\) & 1.048 & 107.1 & 318.3 \\
        \(\sigma_1\) & 1.046 & 66.8 & 105.1 \\
        \(\sigma_2\) & 1.056 & 65.5 & 45.1 \\
        \(q_{12}\) & 1.004 & 2,537.5 & 12,637.6 \\
        \(q_{21}\) & 1.006 & 1,251.7 & 5,873.1 \\
        \(\tau_1\) & 1.016 & 316.6 & 1,082.3 \\
        \(\tau_2\) & 1.016 & 400.5 & 742.1 \\
        \bottomrule
    \end{tabular}
    \end{center}
\end{table*}

\begin{table*}[htbp]
    \caption{Per-parameter MCMC diagnostics for the gene-expression CLE experiment.}
    \label{tab:cle_diagnostics}
    \begin{center}
    \small
    \begin{tabular}{lccc}
        \toprule
        Parameter & \(\hat R\) & Bulk ESS & Tail ESS \\
        \midrule
        \(\rho_1\) & 1.003 & 805.9 & 1,595.1 \\
        \(\rho_2\) & 1.001 & 1,274.5 & 3,142.3 \\
        \(\beta\) & 1.011 & 211.0 & 333.4 \\
        \(\gamma\) & 1.007 & 334.2 & 691.6 \\
        \(\delta\) & 1.007 & 335.3 & 681.7 \\
        \(c\) & 1.018 & 201.7 & 358.9 \\
        \(q_{12}\) & 1.000 & 16,712.8 & 18,545.4 \\
        \(q_{21}\) & 1.000 & 18,588.0 & 19,235.4 \\
        \(\tau_1\) & 1.697 & 6.2 & 11.5 \\
        \(\tau_2\) & 1.003 & 725.2 & 2,003.8 \\
        \bottomrule
    \end{tabular}
    \end{center}
\end{table*}

\begin{table*}[htbp]
    \caption{Per-parameter MCMC diagnostics for the susceptible--infected--recovered experiment.}
    \label{tab:sir_diagnostics}
    \begin{center}
    \small
    \begin{tabular}{lccc}
        \toprule
        Parameter & \(\hat R\) & Bulk ESS & Tail ESS \\
        \midrule
        \(\beta_1\) & 1.019 & 481.9 & 1,013.8 \\
        \(\beta_2\) & 1.054 & 61.6 & 338.2 \\
        \(\gamma_1\) & 1.065 & 66.2 & 485.9 \\
        \(\gamma_2\) & 1.003 & 645.6 & 1,278.9 \\
        \(q_{12}\) & 1.000 & 15,315.1 & 18,870.5 \\
        \(q_{21}\) & 1.001 & 10,209.2 & 14,825.5 \\
        \(\tau_1\) & 1.539 & 7.0 & 41.1 \\
        \(\tau_2\) & 1.298 & 10.8 & 38.8 \\
        \bottomrule
    \end{tabular}
    \end{center}
\end{table*}

\begin{table*}[htbp]
    \caption{Per-parameter MCMC diagnostics for the real-data Lotka--Volterra analysis.}
    \label{tab:real_lv_diagnostics}
    \begin{center}
    \small
    \begin{tabular}{lccc}
        \toprule
        Parameter & \(\hat R\) & Bulk ESS & Tail ESS \\
        \midrule
        \(\alpha_1\) & 1.004 & 442.4 & 987.0 \\
        \(\alpha_2\) & 1.005 & 929.8 & 909.5 \\
        \(\beta_1\) & 1.003 & 466.3 & 1,065.3 \\
        \(\beta_2\) & 1.005 & 871.0 & 888.9 \\
        \(\gamma_1\) & 1.011 & 298.2 & 732.9 \\
        \(\gamma_2\) & 1.007 & 1,160.2 & 2,167.1 \\
        \(\delta_1\) & 1.014 & 296.4 & 493.5 \\
        \(\delta_2\) & 1.004 & 1,563.6 & 3,540.1 \\
        \(\sigma_1\) & 1.027 & 342.6 & 478.5 \\
        \(\sigma_2\) & 1.010 & 276.9 & 552.2 \\
        \(q_{12}\) & 1.002 & 5,023.5 & 9,473.7 \\
        \(q_{21}\) & 1.002 & 3,481.9 & 7,393.4 \\
        \(\tau_1\) & 1.014 & 397.2 & 683.3 \\
        \(\tau_2\) & 1.010 & 183.5 & 344.5 \\
        \bottomrule
    \end{tabular}
    \end{center}
\end{table*}

\clearpage

\subsection{Computing Environment and Software}

\subsubsection{Computing Infrastructure}
\label{app:computing_infrastructure}

All computations used only the CPU on a local PC with an Intel Core i7-1195G7 processor (4 cores, 8 threads) and 32GiB of RAM, running Ubuntu 22.04.5 LTS.

\subsubsection{Software and Existing Assets}
\label{app:software_assets}
We train the local transition densities for FNLE using the NLE implementation in \texttt{sbi} (v0.24.0)~\citep{boelts2025}.
We additionally use the \texttt{Pyro} package (v1.9.1) \citep{bingham2019} for posterior sampling with NUTS.
We compute \(\hat R\) and bulk and tail ESS using \texttt{ArviZ} (v0.22.0)~\citep{kumar2019}.
All three packages are distributed under the Apache License 2.0.
Complete software dependencies and their exact versions are provided in the supplementary material via the accompanying \texttt{uv.lock} file.
The real data used in this study are publicly available in the \texttt{rEDM} R package (BSD 2-Clause license), originally from \citet{veilleux1979}.

\end{document}